%% file: Early-stopping-jrnl.tex
\documentclass[11pt]{article}

\usepackage{fullpage}

\usepackage{amsmath,amssymb,amsthm,graphicx} 
\usepackage{epsfig}
\usepackage{psfrag}
\usepackage{enumerate} 
\usepackage{setspace}
\usepackage{float} 
\usepackage{color}
\usepackage{array}
\usepackage{pgf,tikz}
\usepackage{mathrsfs}
\usepackage{hyperref}
\usepackage{subcaption}

\usepackage[utf8]{inputenc} 
\usepackage[T1]{fontenc}    
\usepackage{hyperref}       
\usepackage{url}            
\usepackage{booktabs}       
\usepackage{amsfonts}       
\usepackage{nicefrac}       
\usepackage{microtype}      

\usetikzlibrary{arrows}
\definecolor{ffqqqq}{rgb}{1.,0.,0.}
\definecolor{xfqqff}{rgb}{0.4980392156862745,0.,1.}

\input{BookMacros}

\input{AMS_commands}

\definecolor{mypink}{rgb}{0.858, 0.188, 0.478}
\definecolor{myred}{rgb}{1, 0, 0}

\renewcommand{\Loss}{\ensuremath{\mathcal{L}}}
\newcommand{\Jloss}{\ensuremath{\mathcal{J}}}
\newcommand{\EmpLoss}{\ensuremath{\Loss_\numobs}}
\newcommand{\EmpJloss}{\ensuremath{\Jloss_\numobs}}
\newcommand{\EmpOp}{G_{\numobs}}
\newcommand{\PopOp}{G}

\newcommand{\fval}{\ensuremath{\theta}}

\newcommand{\fvalstar}{\ensuremath{\fval^*}}

\newcommand{\stsize}{\alpha}

\newcommand{\inprodk}[2]{\langle #1, #2 \rangle_{\Hil}}

\newcommand{\enorm}[1]{\| #1 \|_\numobs}

\newcommand{\knorm}[1]{\|#1\|_{\Hil}}
\newcommand{\ltwo}[1]{\| #1 \|_2}
\newcommand{\Fspace}{\mathcal{F}}

\newcommand{\fvalit}[1]{\ensuremath{\fval^{#1}}}
\newcommand{\HilNorm}[1]{\ensuremath{\| #1\|_{\Hil}}}
\newcommand{\EmpNorm}[1]{\ensuremath{\|#1\|_{\numobs}}}

\newcommand{\DelIt}[1]{\ensuremath{\Delta^{#1}}}
\renewcommand{\strongcon}{\ensuremath{m}}
\newcommand{\lipcon}{\ensuremath{M}}

\newcommand{\zval}{\ensuremath{z}}
\newcommand{\zvalit}[1]{\ensuremath{\zval^{#1}}}
\newcommand{\zvalstar}{\ensuremath{z^*}}

\newcommand{\zstar}{\ensuremath{z^*}}

\newcommand{\Sdiff}[1]{\ensuremath{d^{#1}}}
\newcommand{\Hdiff}[1]{\ensuremath{D}^{#1}}

\renewcommand{\Kmat}{\ensuremath{K}}
\newcommand{\KmatSqrt}{\ensuremath{\sqrt{K}}}

\newcommand{\Exp}{\text{exp}}
\newcommand{\ConstSet}{\mathcal{C}}

\newcommand{\secarg}{\ensuremath{\theta}}

\newcommand{\Tele}[1]{\ensuremath{V^{#1}}}
\newcommand{\Radius}{C_{\Hil}}

\newcommand{\hatf}{\widehat{f}}
\newcommand{\GW}{\mathcal{G}_\numobs}
\newcommand{\GWP}{\mathcal{\widebar{G}}_\numobs}
\newcommand{\EGW}{\mathcal{Z}_\numobs}
\newcommand{\EGWS}{\mathcal{\tilde{Z}}_\numobs}

\newcommand{\Eset}[2]{\mathcal{E}(#1,#2)}

\newcommand{\CrudeTerm}{T}

\newcommand{\E}{\text{e}}

\newcommand{\critquant}{\delta_n}
\newcommand{\popcritquant}{\widebar{\delta}_\numobs}

\newcommand{\sceff}{\gamma}
\newcommand{\scefftil}{\widetilde{\sceff}}

\newcommand{\AnnCon}{c_3}
\newcommand{\A}{\mathcal{A}}

\newcommand{\SubK}{\Hil_{\numobs}}
\newcommand{\AvgEst}{\bar{\fval}}

\newcommand{\Diam}{D}
\newcommand{\PhiB}{B}
\newcommand{\Kerfunc}{\mathbb{K}}
\newcommand{\KerR}{\mathcal{R}}
\newcommand{\PP}{\mathbb{P}}

\newcommand{\HilNormball}{\mathbb{B}_{\Hil}}
\newcommand{\HilNormBall}{\HilNormball}
\newcommand{\Hilstar}{\partial \Hil}
\newcommand{\matchquant}{\xi}

\newcommand{\fcoeff}{\omega}

\DeclareMathOperator*{\argmax}{arg\, max\,}

\renewcommand{\Xspace}{\ensuremath{\mathcal{X}}}
\renewcommand{\Yspace}{\ensuremath{\mathcal{Y}}}
\newcommand{\domX}{\mathcal{X}}
\newcommand{\ExsY}{\ensuremath{\Exs_{Y_1^\numobs}}}

\newcommand{\YDATA}{\ensuremath{\{Y_i\}_{i=1}^\numobs}}

\newcommand{\fit}[1]{\ensuremath{f^{#1}}}
\newcommand{\step}[1]{\ensuremath{\alpha^{#1}}}
\newcommand{\KerFun}{\ensuremath{\mathbb{K}}}
\renewcommand{\Hil}{\ensuremath{\mathscr{H}}}
\newcommand{\FSEQ}{\ensuremath{\{\fit{t}\}_{t=0}^\infty}}
\renewcommand{\muhat}{\mu} 
\renewcommand{\d}{\text{d}}
\newcommand{\pseudinv}{\dagger}
\newcommand{\statdim}{d_{\numobs}}

\newcommand{\level}{\sigma}
\newcommand{\Avgf}[1]{\bar{f}^{#1}}
\newcommand{\opcrit}{\rho_\numobs}

\newcommand{\mMC}{$m$-$M$-condition}
\newcommand{\PhiBC}{$\phi'$-boundedness}

\newcommand{\HANASET}{\ensuremath{\mathbb{S}}}

\newcommand{\EllipseBar}{\ensuremath{\widebar{\Ellipse}}}

\newcommand{\mprobx}{\ensuremath{\mprob_X}}

\newcommand{\divKL}[2]{\|#1,~#2\|_{\text{KL}}}

\newenvironment{hclist}
 {\begin{list}{$\bullet$}
 {\setlength{\topsep}{0in} \setlength{\partopsep}{0in}
  \setlength{\parsep}{0in} \setlength{\itemsep}{\parskip}
  \setlength{\leftmargin}{0.15in} \setlength{\rightmargin}{0.08in}
  \setlength{\listparindent}{0in} \setlength{\labelwidth}{0.08in}
  \setlength{\labelsep}{0.1in} \setlength{\itemindent}{0in}}}
 {\end{list}}

\renewcommand{\thefootnote}{\fnsymbol{footnote}}

\makeatletter
\long\def\@makecaption#1#2{
        \vskip 0.8ex
        \setbox\@tempboxa\hbox{\small {\bf #1:} #2}
        \parindent 1.5em  
        \dimen0=\hsize
        \advance\dimen0 by -3em
        \ifdim \wd\@tempboxa >\dimen0
                \hbox to \hsize{
                        \parindent 0em
                        \hfil 
                        \parbox{\dimen0}{\def\baselinestretch{0.96}\small
                                {\bf #1.} #2
                                } 
                        \hfil}
        \else \hbox to \hsize{\hfil \box\@tempboxa \hfil}
        \fi
        }
\makeatother

\begin{document}

\begin{center}

{\bf{\Large{Early stopping for kernel boosting algorithms: 
A general analysis with localized complexities}}}

\vspace{.2in} 

{\large{
\begin{tabular}{ccccc}
Yuting Wei$^{1*}$ && Fanny Yang$^{2*}$
 &&  
Martin J. Wainwright$^{1, 2}$ 
\end{tabular}
}}

\vspace{.2in} \today

\begin{tabular}{c}
 Department of Statistics$^1$, and \\ Department of Electrical
 Engineering and Computer Sciences$^2$ \\ UC Berkeley, Berkeley,
 CA 94720
\end{tabular}

\begin{abstract}
  Early stopping of iterative algorithms is a widely-used form of
  regularization in statistics, commonly used in conjunction with
  boosting and related gradient-type algorithms. Although consistency
  results have been established in some settings, such estimators are
  less well-understood than their analogues based on penalized
  regularization.  In this paper, for a relatively broad class of loss
  functions and boosting algorithms (including $L^2$-boost, LogitBoost
  and AdaBoost, among others), we exhibit a direct connection between
  the performance of a stopped iterate and the localized Gaussian
  complexity of the associated function class.  This connection allows
  us to show that local fixed point analysis of Gaussian or Rademacher
  complexities, now standard in the analysis of penalized estimators,
  can be used to derive optimal stopping rules.  We derive such
  stopping rules in detail for various kernel classes, and illustrate
  the correspondence of our theory with practice for Sobolev kernel
  classes.
  \end{abstract}

\end{center}

\footnotetext[1]{Yuting Wei and Fanny Yang contributed equally to this work.}

\renewcommand*{\thefootnote}{\arabic{footnote}}

\section{Introduction}
While non-parametric models offer great flexibility, they can also
lead to overfitting, and thus poor generalization performance.  For
this reason, it is well-understood that procedures for fitting
non-parametric models must involve some form of regularization.  When
models are fit via a form of empirical risk minimization, the most
classical form of regularization is based on adding some type of
penalty to the objective function.  An alternative form of
regularization is based on the principle of \emph{early stopping}, in
which an iterative algorithm is run for a pre-specified number of
steps, and terminated prior to convergence.

While the basic idea of early stopping is fairly old
(e.g.,~\cite{Strand74,AndrerssenPrenter81,Wahba87}), recent years have
witnessed renewed interests in its properties, especially in the
context of boosting algorithms and neural network training
(e.g.,~\cite{prechelt1998early,caruana2001overfitting}). Over the past
decade, a line of work has yielded some theoretical insight into early
stopping, including works on classification error for boosting
algorithms~\cite{Bartlett07, Freund97, Jiang04, Mason99,
  Yao07,ZhangYu05}, $L^2$-boosting algorithms for
regression~\cite{BuhlmannYu03,BuehlHot07}, and similar gradient
algorithms in reproducing kernel Hilbert spaces
(e.g.~\cite{Caponnetto06,CaponnettoYao06,DeVito10, Yao07,
  RasWaiYu14}).  A number of these papers establish consistency
results for particular forms of early stopping, guaranteeing that the
procedure outputs a function with statistical error
that converges to zero as the
sample size increases.  On the other hand, there are relatively few
results that actually establish \emph{rate optimality} of an early
stopping procedure, meaning that the achieved error matches known
statistical minimax lower bounds.  To the best of our knowledge,
B\"uhlmann and Yu~\cite{BuhlmannYu03} were the first to prove optimality
for early stopping of $L^2$-boosting as applied to spline classes,
albeit with a rule that was not computable from the data.  Subsequent
work by Raskutti et al.~\cite{RasWaiYu14} refined this analysis of
$L^2$-boosting for kernel classes and first established an important
connection to the localized Rademacher complexity; see also the
related work~\cite{Yao07, Ros15, Camo16} with rates for particular
kernel classes.

More broadly, relative to our rich and detailed understanding of
regularization via penalization (e.g., see the books~\cite{GyorfiEtal,
  vanderVaart96,vandeGeer00,Wai17} and papers~\cite{Bar05,Kolt06} for
details), our understanding of early stopping regularization is not as
well developed.  Intuitively, early stopping should depend on the same
bias-variance tradeoffs that control estimators based on penalization.
In particular, for penalized estimators, it is now well-understood
that complexity measures such as the \emph{localized Gaussian width},
or its Rademacher analogue, can be used to characterize their
achievable rates~\cite{Bar05,Kolt06,vandeGeer00,Wai17}.  Is such a
general and sharp characterization also possible in the context of
early stopping?

The main contribution of this paper is to answer this question in the
affirmative for the early stopping of boosting algorithms for a
certain class of regression and classification problems involving
functions in reproducing kernel Hilbert spaces (RKHS). A standard way
to obtain a good estimator or classifier is through minimizing some
penalized form of loss functions of which the method of kernel ridge
regression~\cite{Wahba90} is a popular choice.  Instead, we consider
an iterative update involving the kernel that is derived from a greedy
update.  Borrowing tools from empirical process theory, we are able to
characterize the ``size'' of the effective function space explored by
taking $T$ steps, and then to connect the resulting estimation error
naturally to the notion of localized Gaussian width defined with
respect to this effective function space.  This leads to a principled
analysis for a broad class of loss functions used in practice,
including the loss functions that underlie the $L^2$-boost, LogitBoost
and AdaBoost algorithms, among other procedures.

The remainder of this paper is organized as follows.  In
Section~\ref{SecBackground}, we provide background on boosting methods
and reproducing kernel Hilbert spaces, and then introduce the updates
studied in this paper.  Section~\ref{SecMain} is devoted to statements
of our main results, followed by a discussion of their consequences
for particular function classes in Section~\ref{SecExamples}.  We
provide simulations that confirm the practical effectiveness of our
stopping rules, and show close agreement with our theoretical
predictions.  In Section~\ref{SecProofs}, we provide the proofs of our
main results, with certain more technical aspects deferred to the
appendices.

\section{Background and problem formulation}
\label{SecBackground}

The goal of prediction is to learn a function that maps
\emph{covariates} $x \in \Xspace$ to \emph{responses} $y \in \Yspace$.
In a regression problem, the responses are typically real-valued,
whereas in a classification problem, the responses take values in a
finite set.  In this paper, we study both regression
($\Yspace = \real$) and classification problems (e.g.,
$\Yspace = \{-1, +1 \}$ in the binary case). 
Our primary focus is on the case of
\emph{fixed design}, in which we observe a collection of $\numobs$
pairs of the form $ \{(x_i, Y_i) \}_{i=1}^\numobs$, where each $x_i
\in \Xspace$ is a fixed covariate, whereas $Y_i \in \Yspace$ is a
random response drawn independently from a distribution $\PP_{Y|x_i}$
which depends on $x_i$.  Later in the paper, we also discuss the
consequences of our results for the case of random design, where the
$(X_i, Y_i)$ pairs are drawn in an i.i.d. fashion from the joint
distribution $\PP = \mprobx \PP_{Y|X}$ for some distribution $\mprobx$ on the 
covariates.

In this section, we provide some necessary background on a
gradient-type algorithm which is often referred to as \emph{boosting}
algorithm.
We also discuss briefly about the reproducing kernel
Hilbert spaces before turning to a precise formulation of the problem
that is studied in this paper.

\subsection{Boosting and early stopping}

Consider a cost function $\phi: \real \times \real \rightarrow [0,
  \infty)$, where the non-negative scalar $\phi(y, \secarg)$ denotes
  the cost associated with predicting $\secarg$ when the true response
  is $y$.  Some common examples of loss functions $\phi$ that we
  consider in later sections include:
\bcar
  \item the \emph{least-squares loss} $\phi(y, \secarg) \defn
    \frac{1}{2}(y-\secarg)^2$ that underlies $L^2$-boosting~\cite{BuhlmannYu03},
  \item the \emph{logistic regression loss} $\phi(y, \secarg) = \ln(1+
     e^{-y \secarg })$ that underlies the LogitBoost algorithm~\cite{friedman00additive,Friedman01}, and
\item the \emph{exponential loss} $\phi(y, \secarg) = \exp( -y
  \secarg)$ that underlies the AdaBoost algorithm~\cite{Freund97}.
\ecar
The least-squares loss is typically used for regression problems
(e.g.,~\cite{BuhlmannYu03,Caponnetto06,CaponnettoYao06,DeVito10,
  Yao07, RasWaiYu14}), whereas the latter two losses are frequently
used in the setting of binary classification
(e.g.,~\cite{Freund97,Mason99,Friedman01}).

Given some loss function $\phi$, we define the \emph{population cost
  functional} $f \mapsto \Loss(f)$ via
\begin{align}
  \label{EqnPopLoss}
\Loss(f) \defn \ExsY \Big[ \frac{1}{\numobs} \sum_{i=1}^\numobs
  \phi\big(Y_i, f(x_i) \big) \Big].
\end{align}
Note that with the covariates $\{x_i\}_{i=1}^\numobs$ fixed, the
functional $\Loss$ is a non-random object.  Given some function space
$\Fclass$, the optimal function\footnote{As clarified in the sequel,
  our assumptions guarantee uniqueness of $\fstar$.} 
  minimizes the
population cost functional---that is
  \begin{align}
  \label{EqnDefnFstar}
    \fstar & \defn  \arg \min_{f \in \Fclass} \Loss(f).
  \end{align}
As a standard example, when we adopt the least-squares loss $\phi(y,
\secarg) = \frac{1}{2}(y-\secarg)^2$, the population minimizer
$\fstar$ corresponds to the conditional expectation $x \mapsto \Exs [Y
  \mid x]$.

Since we do not have access to the population distribution of the
responses however, the computation of $\fstar$ is impossible.  Given
our samples $\YDATA$, we consider instead some procedure applied to
the \emph{empirical loss}
\begin{align}
\label{EqnEmpLoss}
\EmpLoss(f) & \defn \frac{1}{\numobs} \sum_{i=1}^\numobs \phi(Y_i,
f(x_i)),
\end{align}
where the population expectation has been replaced by an empirical
expectation. For example, when $\EmpLoss$ corresponds to the log
likelihood of the samples with $\phi(Y_i, f(x_i)) = \log [\PP(Y_i;
  f(x_i))]$, direct unconstrained minimization of $\EmpLoss$ would
yield the maximum likelihood estimator.

It is well-known that direct minimization of $\EmpLoss$ over a
sufficiently rich function class $\Fclass$ may lead to overfitting.
There are various ways to mitigate this phenomenon, among which the
most classical method is to minimize the sum of the empirical loss
with a penalty regularization term.  Adjusting the weight on the
regularization term allows for trade-off between fit to the data, and
some form of regularity or smoothness in the fit.  The behavior of
such penalized of regularized estimation methods is now quite well
understood (for instance, see the
books~\cite{GyorfiEtal,vanderVaart96,vandeGeer00,Wai17} and
papers~\cite{Bar05,Kolt06} for more details).

In this paper, we study a form of \emph{algorithmic regularization},
based on applying a gradient-type algorithm to $\EmpLoss$ but then
stopping it ``early''---that is, after some fixed number of steps.
Such methods are often referred to as \emph{boosting algorithms},
since they involve ``boosting'' or improve the fit of a function via a
sequence of additive updates (see
e.g. \cite{schapire90strength,Freund97,Breiman98,breiman1999prediction,schapire2003boosting}). Many
boosting algorithms, among them AdaBoost~\cite{Freund97},
$L^2$-boosting~\cite{BuhlmannYu03} and
LogitBoost~\cite{friedman00additive,Friedman01}, can be understood as
forms of functional gradient methods~\cite{Mason99,Friedman01}; see
the survey paper~\cite{BuehlHot07} for further background on boosting.
The way in which the number of steps is chosen is referred to as a
stopping rule, and the overall procedure is referred to as \emph{early
  stopping} of a boosting algorithm.


\begin{figure}[h]
  \begin{center}
    \begin{tabular}{ccc}
      \widgraph{.45\textwidth}{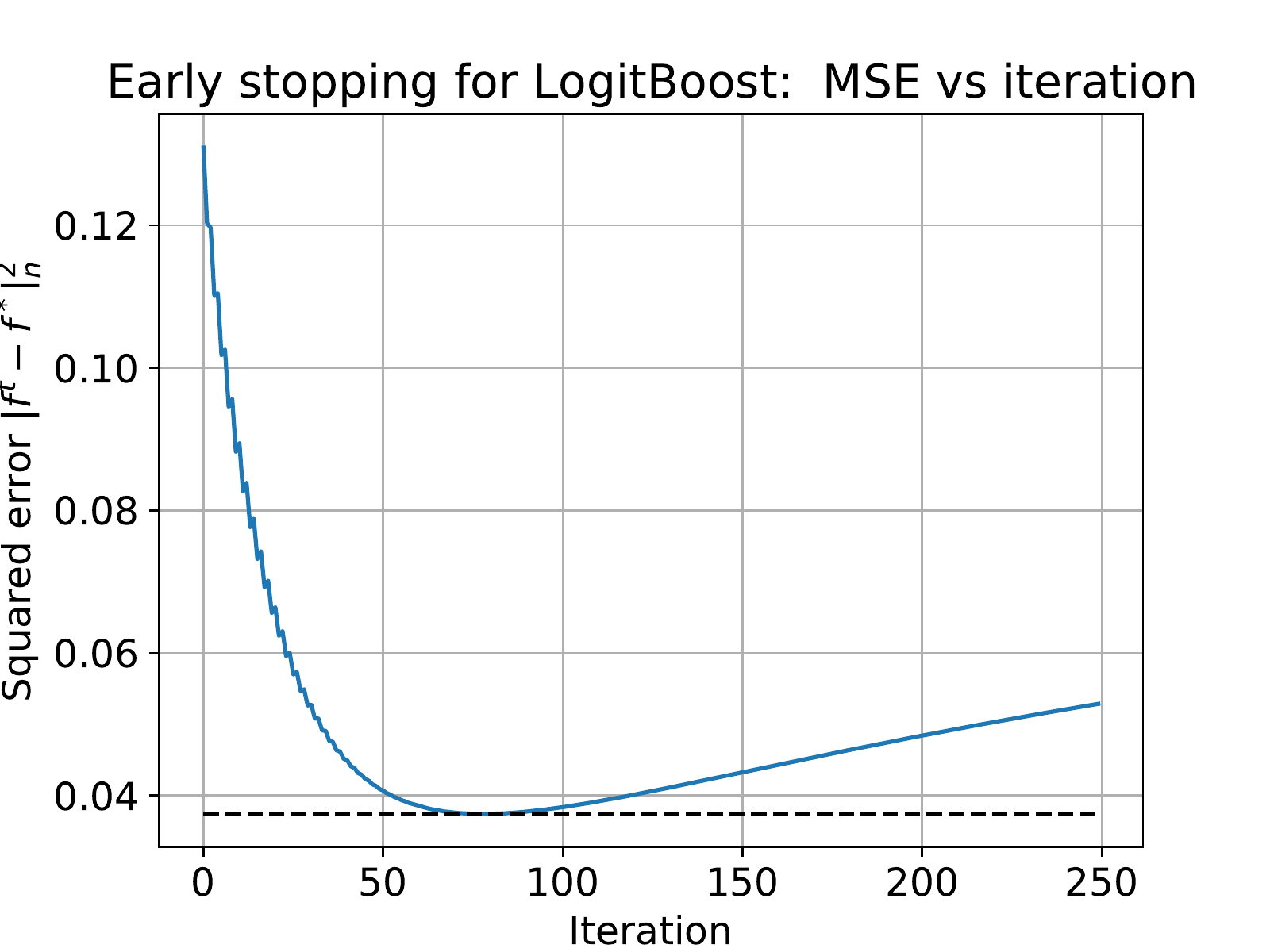} &&
      \widgraph{.45\textwidth}{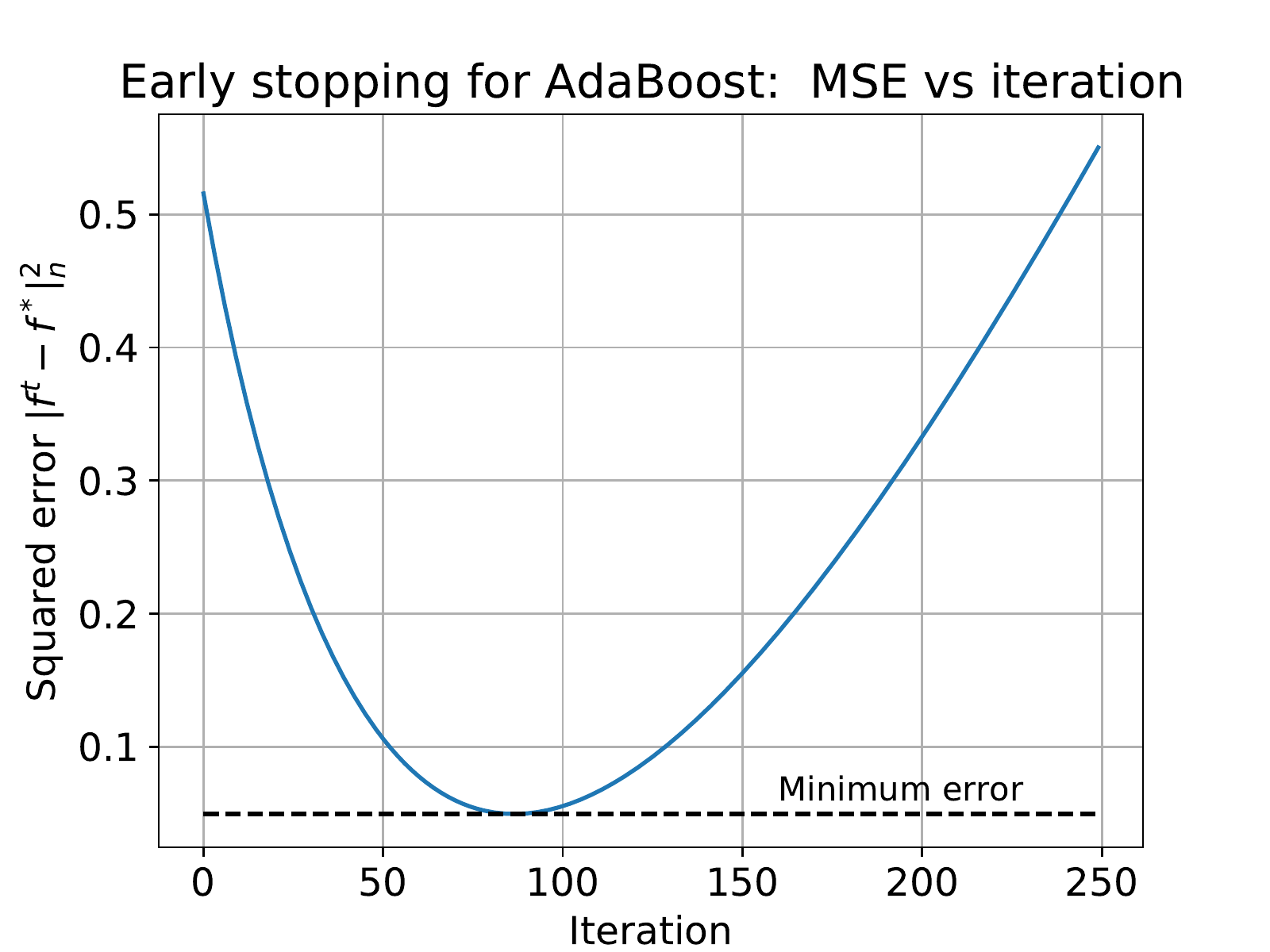}\\
      (a) & & (b)
    \end{tabular}
    \caption{Plots of the squared error
      $\|\fit{t} - \fstar\|_\numobs^2 = \frac{1}{\numobs}
      \sum_{i=1}^\numobs (\fit{t}(x_i) - \fstar(x_i))^2$ versus the
      iteration number $t$ for (a) LogitBoost using a first-order
      Sobolev kernel (b) AdaBoost using the same first-order Sobolev
      kernel $\KerFun(x, x') = 1 + \min(x, x')$ which generates a
      class of Lipschitz functions (splines of order one).  Both plots
      correspond to a sample size $n=100$.}
    \label{FigEarly}
  \end{center}
\end{figure}

In more detail, a broad class of boosting algorithms~\cite{Mason99}
generate a sequence $\FSEQ$ via updates of the form
\begin{align}
  \label{EqnGeneralUpdate}
\fit{t+1} & = \fit{t} - \step{t} g^t \quad \text{ with } \quad g^t \propto \argmax_{\|d\|_\Fclass \leq 1} \inprod{\nabla \EmpLoss(\fit{t})}{d(x_1^\numobs)},
\end{align}
where the scalar $\{\step{t}\}_{t=0}^{\infty}$ is a sequence of step
sizes chosen by the user, the constraint $\|d\|_\Fclass \leq 1$
defines the unit ball in a given function class $\Fclass$, $\nabla
\EmpLoss(f) \in \real^\numobs$ denotes the gradient taken at the
vector $\big(f(x_1), \ldots, f(x_\numobs))$, and $\inprod{h}{g}$ is
the usual inner product between vectors $h, g \in \real^\numobs$.  For
non-decaying step sizes and a convex objective $\EmpLoss$, running
this procedure for an infinite number of iterations will lead to a
minimizer of the empirical loss, thus causing overfitting.  In order
to illustrate this phenomenon, Figure~\ref{FigEarly} provides plots of
the squared error $\|f^t - \fstar\|_\numobs^2 \defn \frac{1}{\numobs}
\sum_{i=1}^\numobs \big(f^t(x_i) - \fstar(x_i) \big)^2$ versus the
iteration number, for LogitBoost in panel (a) and AdaBoost in panel
(b).  See Section~\ref{SecSimulations} for more details on exactly how
these experiments were conducted.

In the plots in Figure~\ref{FigEarly}, the dotted line indicates the
minimum mean-squared error $\opcrit^2$ over all iterates of that
particular run of the algorithm.  Both plots are qualitatively
similar, illustrating the existence of a ``good'' number of iterations
to take, after which the MSE greatly increases.  Hence a natural
problem is to decide at what iteration $T$ to stop such that the
iterate $f^T$ satisfies bounds of the form
\begin{align}
  \label{EqnHighLevelBounds}
\Loss(f^T) - \Loss(f^*) & \precsim \opcrit^2 \quad \mbox{and}
\quad \|f^T - \fstar\|_\numobs^2 \precsim \opcrit^2
\end{align}
with high probability. Here $f(n) \precsim g(n)$
indicates that $f(n) \leq cg(n)$ for some universal constant $c \in (0, \infty)$.
The main results of this paper provide a
stopping rule $T$ for which bounds of the
form~\eqref{EqnHighLevelBounds} do in fact hold with high probability
over the randomness in the observed responses.

Moreover, as shown by our later results, under suitable regularity
conditions, the expectation of the minimum squared error $\opcrit^2$
is proportional to the \emph{statistical minimax risk} $\inf_{\fhat}
\sup_{f \in \Fspace} \Exs[\Loss(\fhat) - \Loss(f)]$, where the infimum
is taken over all possible estimators $\fhat$. Note that the minimax
risk provides a fundamental lower bound on the performance of any
estimator uniformly over the function space $\Fspace$.  Coupled with
our stopping time guarantee~\eqref{EqnHighLevelBounds}, we are
guaranteed that our estimate achieves the minimax risk up to constant
factors.  As a result, our bounds are unimprovable in general (see
Corollary~\ref{CorOptimal}).

\subsection{Reproducing Kernel Hilbert Spaces}

The analysis of this paper focuses on algorithms with the
update~\eqref{EqnGeneralUpdate} when the function class $\Fclass$ is a
reproducing kernel Hilbert space $\Hil$ (RKHS, see standard sources
~\cite{Wahba90,Gu02,Scholkopf02,BerTho04}), consisting of functions mapping a domain $\Xspace$ to
the real line $\real$.  Any RKHS is defined by a bivariate symmetric
\emph{kernel function}
$\KerFun: \Xspace \times \Xspace \rightarrow \real$ which is required
to be positive semidefinite, i.e. for any integer $N\geq 1$ and a
collection of points $\{x_j\}_{j=1}^N$ in $\Xspace$, the matrix
$[\KerFun(x_i, x_j)]_{ij} \in \real^{N \times N}$ is positive
semidefinite.

The associated RKHS is the closure of the linear span of 
functions in the form $f(\cdot) =
\sum_{j \geq 1} \omega_j \KerFun(\cdot, x_j)$, where
$\{x_j\}_{j=1}^\infty$ is some collection of points in $\Xspace$, and
$\{\omega_j \}_{j=1}^\infty$ is a real-valued sequence.  For two
functions $f_1, f_2 \in \Hil$ which can be expressed as a finite sum
$f_1(\cdot) = \sum_{i=1}^{\ell_1}\alpha_i \Kerfunc(\cdot, x_i)$ and
$f_2(\cdot) = \sum_{j=1}^{\ell_2} \beta_j \Kerfunc(\cdot, x_j)$, the
inner product is defined as $\inprodk{f_1}{f_2} = \sum_{i=1}^{\ell_1}
\sum_{j=1}^{\ell_2} \alpha_i \beta_j \Kerfunc(x_i, x_j)$ with induced
norm $\knorm{f_1}^2 = \sum_{i=1}^{\ell_1} \alpha_i^2\Kerfunc(x_i,
x_i)$.  For each $x \in \Xspace$, the function $\KerFun(\cdot , x)$
belongs to $\Hil$, and satisfies the reproducing relation
\begin{align}
  \inprod{f}{\KerFun(\cdot, x)}_{\Hil} = f(x) ~~\text{ for all } f \in
  \Hil.
\end{align}

Moreover, when the covariates $X_i$ are drawn i.i.d. from a distribution $\mprobx$
with domain $\domX$ we can invoke Mercer's theorem which states that
 any function in $\Hil$ can be represented as
\begin{align}
\label{EqnMercer}
\Kerfunc(x,x') & = \sum_{k=1}^\infty \mu_k \phi_k(x) \phi_k(x'),
\end{align}
where $\mu_1 \geq \mu_2 \geq \dots \geq 0$ are the \emph{eigenvalues}
of the kernel function $\Kerfunc$ and $\{\phi_k\}_{k=1}^\infty$ are
eigenfunctions of $\Kerfunc$ which form an orthonormal basis of
$L^2(\domX, \mprobx)$ 
with the inner product $\langle f, g\rangle \defn \int_{\domX} f(x)g(x)
\d \mprobx(x)$.
We refer the
reader to the standard
sources~\cite{Wahba90,Gu02,Scholkopf02,BerTho04} for more details on
RKHSs and their properties.

Throughout this paper, we assume that the kernel function is uniformly
bounded, meaning that there is a constant $L$ such that \mbox{$\sup_{x
    \in \Xspace} \KerFun(x,x) \leq L$.}  Such a boundedness condition
holds for many kernels used in practice, including the Gaussian,
Laplacian, Sobolev, other types of spline kernels, as well as any
trace class kernel with trigonometric eigenfunctions.  By rescaling
the kernel as necessary, we may assume without loss of generality that
$L = 1$.  As a consequence, for any function $f$ such that
$\HilNorm{f} \leq r$, we have by the reproducing relation that
\begin{align*}
\|f\|_\infty = \sup_{x} \inprodk{f}{\Kerfunc(\cdot, x)} \leq
\HilNorm{f} \sup_x \HilNorm{\Kerfunc(\cdot, x)} \leq r.
\end{align*}

Given samples $\{(x_i,y_i)\}_{i=1}^\numobs$, by the representer
theorem~\cite{KimWah71}, it is sufficient to restrict ourselves to the
linear subspace $\SubK = \widebar{\text{span}}\{\Kerfunc(\cdot,
x_i)\}_{i=1}^\numobs$, for which all $f\in \SubK$ can be expressed as
\begin{align}
\label{EqnKernelRep}
f & = \frac{1}{\sqrt{\numobs}} \sum_{i=1}^\numobs \fcoeff_i
\Kerfunc(\cdot, x_i)
\end{align}
for some coefficient vector $\fcoeff \in \real^\numobs$.  Among those
functions which achieve the infimum in expression \eqref{EqnPopLoss},
let us define $\fstar$ as the one with the minimum Hilbert norm.  This
definition is equivalent to restricting $\fstar$ to be in the linear
subspace $\SubK$.

\subsection{Boosting in kernel spaces}
\label{SecBoostKer}

For a finite number of covariates $x_i$ from $i =1 \dots \numobs$, let
us define the \emph{normalized kernel matrix} $\Kmat \in
\real^{\numobs \times \numobs}$ with entries $\Kmat_{ij} =
\Kerfunc(x_i, x_j)/\numobs$.  
Since we can restrict the minimization of $\EmpLoss$ and $\Loss$ from
$\Hil$ to the subspace $\SubK$ w.l.o.g., using
expression~\eqref{EqnKernelRep} we can then write the function value
vectors $f(x_1^\numobs) \defn (f(x_1), \dots, f(x_\numobs))$ as
$f(x_1^\numobs) = \sqrt{\numobs}K \omega$. As there is a one-to-one
correspondence between the $\numobs$-dimensional vectors
$f(x_1^\numobs) \in \real^\numobs$ and the corresponding function
$f \in \SubK$ in $\Hil$ by the representer theorem, minimization of an
empirical loss in the subspace $\SubK$ essentially becomes the
$\numobs$-dimensional problem of fitting a response vector $y$ over
the set $\range(K)$. In the sequel, all updates will thus be performed
on the function value vectors $f(x_1^\numobs)$.

With a change of variable $d(x_1^\numobs) = \sqrt{\numobs}\sqrt{\Kmat}
z$ we then have
\begin{align*}
d^t(x_1^\numobs) &\defn \argmax \limits_{\substack{\|d\|_{\Hil} \leq 1
    \\ d \in \range(K)}} \inprod{\nabla \EmpLoss(f^t)}{d(x_1^\numobs)} 
= \frac{\sqrt{\numobs} K \nabla \EmpLoss(f^t)}{\sqrt{\nabla
    \EmpLoss(f^t) K \nabla \EmpLoss(f^t)}}.
\end{align*}
In this paper, we study the choice $g^t = \inprod{\nabla
  \EmpLoss(f^t)}{d^t(x_1^\numobs)} d^t$ in the boosting
update~\eqref{EqnGeneralUpdate}, so that the function value iterates
take the form
\begin{align}
  \label{EqnUpdate}
\fit{t+1}(x_1^\numobs) & = \fit{t}(x_1^\numobs) - \step{} \numobs
\Kmat \nabla \EmpLoss(\fit{t}),
\end{align}
where $\step{} > 0$ is a constant stepsize choice. Choosing
$\fit{0}(x_1^\numobs) = 0$ ensures that all iterates
$\fit{t}(x_1^\numobs)$ remain in the range space of $\Kmat$.

In this paper we consider the following three error measures for an 
estimator $\hatf$:
\begin{align*}
  L^2(\PP_\numobs) \text{ norm: }~~~
  &\|\hatf - \fstar\|_\numobs^2 = \frac{1}{\numobs} \sum_{i=1}^\numobs
  \big(\hatf(x_i) - \fstar(x_i) \big)^2,\\
  L^2(\mprobx) \text{ norm: }~~~
  &\|\hatf - \fstar\|_2^2 \defn \Exs \big(\hatf(X) - \fstar(X) \big)^2,\\
  \text{Excess risk: }~~~ &\Loss(\hatf) - \Loss(f^*),
\end{align*}
where the expectation in the $L^2(\mprobx)$-norm is taken over random
covariates $X$ which are independent of the samples $(X_i, Y_i)$ used
to form the estimate $\hatf$.  Our goal is to propose a stopping time
$T$ such that the averaged function $\hatf = \frac{1}{T} \sum_{t=1}^T
f^t$ satisfies bounds of the type~\eqref{EqnHighLevelBounds}.  We
begin our analysis by focusing on the empirical $L^2(\PP_\numobs)$
error, but as we will see in Corollary~\ref{CorRandDesign}, bounds on
the empirical error are easily transformed to bounds on the population
$L^2(\mprobx)$ error.  Importantly, we exhibit such bounds with a
statistical error term $\delta_\numobs$ that is specified by the
\emph{localized Gaussian complexity} of the kernel class.


\section{Main results}
\label{SecMain}
We now turn to the statement of our main results, beginning with the
introduction of some regularity assumptions.

\subsection{Assumptions}

Recall from our earlier set-up that we differentiate between the
empirical loss function $\EmpLoss$ in expression~\eqref{EqnEmpLoss},
and the population loss $\Loss$ in expression~\eqref{EqnPopLoss}.
Apart from assuming differentiability of both functions, all of our
remaining conditions are imposed on the population loss.  Such conditions
at the population level are weaker than their analogues at the empirical
level.

For a given radius $r > 0$, let us define the Hilbert ball around the
optimal function $\fstar$ as
\begin{align}
  \HilNormball(\fstar, r) \defn \{f \in \Hil \mid \HilNorm{f -\fstar}
  \leq r\}.
\end{align}
Our analysis makes particular use of this ball defined for the radius
$\Radius^2 \defn 2\max\{\HilNorm{\fstar}^2,~ 32, \sigma^2\}$ where the effective noise level $\sigma$ is defined in the sequel. 

We assume that the population loss is $\strongcon$-strongly
convex and $\lipcon$-smooth over $\HilNormball(\fstar, 2\Radius)$, meaning that the \\
$\text{\bf{\mMC}:}$
\begin{align*}
  \frac{\strongcon}{2} \|f - g\|_\numobs^2 \leq
   \Loss(f) - \Loss(g) - &\inprod{\nabla \Loss(g)}{f(x_1^n) - g(x_1^n)} \\
    &\hspace*{1.5cm}\leq \frac{\lipcon}{2} \|f - g\|_\numobs^2 
\end{align*}
holds for all $f, g \in \HilNormball(\fstar, 2 \Radius)$ and all
design points $\{x_i\}_{i=1}^\numobs$.  In addition, we assume that
the function $\phi$ is $M$-Lipschitz in its second argument over the
interval $\theta \in [\min \limits_{i \in [\numobs]} \fstar(x_i) -
  2\Radius, \max \limits_{i \in [\numobs]} \fstar(x_i) + 2\Radius]$.
To be clear, here $\nabla \Loss(g)$ denotes the vector in
$\real^\numobs$ obtained by taking the gradient of $\Loss$ with
respect to the vector $g(x_1^\numobs)$.  It can be verified by a
straightforward computation that when $\Loss$ is induced by the
least-squares cost $\phi(y, \theta) = \frac{1}{2} (y - \theta)^2$, the
\mMC~ holds for $\strongcon = \lipcon = 1$.  The logistic and
exponential loss satisfy this condition (see supp. material), where it
is key that we have imposed the condition \emph{only locally} on the
ball $\HilNormball(\fstar, 2 \Radius)$.

In addition to the least-squares cost, our theory also applies to losses
$\Loss$ induced by scalar functions $\phi$ that satisfy the\\
$\text{\bf{\PhiBC:}}$
\begin{align*}
&\max_{i=1, \ldots, \numobs} \left|\frac{\partial
  \phi(y, \theta)}{\partial \theta}\right|_{\theta = f(x_i)} \leq
\PhiB \\
&\hspace*{1cm}\mbox{ for all $f \in \HilNormBall(\fstar, 2 \Radius)$ and
} y\in\Yspace.
\end{align*}
This condition holds with $\PhiB = 1$ for the logistic loss for all
$\Yspace$, and $\PhiB = \exp(2.5 \Radius)$ for the exponential loss for
binary classification with $\Yspace = \{-1,1\}$, using our kernel
boundedness condition.  Note that whenever this condition holds with
some finite $\PhiB$, we can always rescale the scalar loss $\phi$ by $1/B$ so
that it holds with $\PhiB = 1$, and we do so in order to simplify
the statement of our results.

\subsection{Upper bound in terms of localized Gaussian width}

Our upper bounds involve a complexity measure known as the localized
Gaussian width.  In general, Gaussian widths are widely used to obtain
risk bounds for least-squares and other types of $M$-estimators.  In
our case, we consider Gaussian complexities for ``localized'' sets of
the form
\begin{align}
\Ellipse_\numobs(\delta, 1) \defn \Big \{f-g \mid 
\|f - g\|_\Hil \leq 1, ~\|f - g\|_\numobs \leq \delta \Big \}
\end{align}
with $f,g\in\Hil$.
The Gaussian complexity localized at scale $\delta$ is given by
\begin{align}
\label{EqnGW}
\GW \big(\Ellipse_\numobs(\delta, 1) \big) & \defn \Exs \Big[ \sup_{g \in
    \Ellipse_\numobs(\delta, 1)} \frac{1}{\numobs} \sum_{i=1}^\numobs w_i
  g(x_i) \Big],
\end{align}
where $(w_1, \ldots, w_\numobs)$ denotes an i.i.d. sequence of
standard Gaussian variables.

An essential quantity
in our theory is specified by a certain fixed point
equation that is now standard in empirical process
theory~\cite{vandeGeer00, Bar05, Kolt06, RasWaiYu14}.
Let us define
the \emph{effective noise level}
\begin{align}
\label{EqnLevel}
  \level & \defn \begin{cases} \min \Big \{t \, \mid \, \max
    \limits_{i=1, \ldots, \numobs} \Exs[e^{((Y_i - f^*(x_i))^2/t^2)}]
    < \infty \Big \} ~~\mbox{for L.S.} \\
    4 \, (2\lipcon+1)(1 + 2\Radius) ~~\mbox{for $\phi'$-bounded losses.}
  \end{cases}
\end{align}
The \emph{critical radius} $\critquant$ is the smallest positive scalar such
that
\begin{align}
\label{Eqn75perChocolate}
\frac{\GW(\Ellipse_\numobs(\delta,1))}{\delta} ~\leq~ \frac{\delta}{\level}.
\end{align}
We note that past work on localized Rademacher and Gaussian
complexity~\cite{Men02,Bar05} guarantees that there exists a unique
$\critquant > 0$ that satisfies this condition, so that our definition
is sensible.


\vspace*{0.2cm}

\subsubsection{Upper bounds on excess risk and empirical $L^2(\mprob_\numobs)$-error}

With this set-up, we are now equipped to state our main theorem.  It
provides high-probability bounds on the excess risk and
$L^2(\mprob_\numobs)$-error of the estimator $\Avgf{T} \defn
\frac{1}{T} \sum_{t=1}^T f^t$ defined by averaging the $T$ iterates of
the algorithm.  It applies to both the least-squares cost function,
and more generally, to any loss function satisfying the \mMC~and the
$\phi'$-boundedness condition.

\begin{theos}
\label{ThmGodWell}
Suppose that the sample size $\numobs$ large enough such that
$\critquant \leq \frac{\lipcon}{\strongcon}$, and we compute the
sequence $\{f^t\}_{t=0}^\infty$ using the update~\eqref{EqnUpdate}
with initialization $f^0 = 0$ and any step size \mbox{$\stsize \in (0,
  \min\{ \frac{1}{\lipcon}, \lipcon \}]$}. Then for any iteration $T
\in \big\{ 0,1,\ldots \lfloor \frac{\strongcon}{8 \lipcon
  \critquant^2} \rfloor \big \}$, the averaged function estimate
$\Avgf{T}$ satisfies the bounds
\begin{subequations}
\begin{align}
  \label{EqnMasterUnwrapped}
  \Loss(\Avgf{T}) - \Loss(\fstar) & \leq C\lipcon
  \Big(\frac{1}{\stsize \strongcon T} + \frac{\critquant^2}{\strongcon^2} \Big), \quad \mbox{and} \\
  \label{EqnEnormBound}
  \enorm{\Avgf{T} - \fstar}^2 & \leq C\Big( \frac{1}{\stsize
    \strongcon T} + \frac{\critquant^2}{\strongcon^2} \Big),
\end{align}
\end{subequations}
where both inequalities hold with probability at least \mbox{$1 - c_1
  \exp (- C_2\frac{m^2\numobs \critquant^2}{ \level^2})$.}
\end{theos}
\noindent We prove Theorem~\ref{ThmGodWell} in
Section~\ref{SecProofThm}.

A few comments about the constants in our statement: in all cases,
constants of the form $c_j$ are universal, whereas the capital $C_j$
may depend on parameters of the joint distribution and population loss
$\Loss$.  In Theorem~\ref{ThmGodWell}, we have the explicit value $C_2
= \{\frac{\strongcon^2}{\level^2},1\}$ and $C^2$ is proportional to
the quantity $2 \max\{\HilNorm{\fstar}^2,~ 32,~ \sigma^2\}$.  While
inequalities~\eqref{EqnMasterUnwrapped} and~\eqref{EqnEnormBound} are
stated as high probability results, similar bounds for expected loss
(over the response $y_i$, with the design fixed) can be obtained by a
simple integration argument.

In order to gain intuition for the claims in the theorem, note that
apart from factors depending on $(\strongcon, \lipcon)$, the first
term $\frac{1}{\stsize \strongcon T}$ dominates the second term
$\frac{\critquant^2}{\strongcon^2}$ whenever $T \lesssim
1/\critquant^2$.  Consequently, up to this point, taking further
iterations reduces the upper bound on the error. This reduction
continues until we have taken of the order $1/\critquant^2$ many
steps, at which point the upper bound is of the order $\critquant^2$.

More precisely, suppose that we perform the updates with step size
$\stsize = \frac{\strongcon}{\lipcon}$; then, after a total number of
$\tau \defn \frac{1}{\critquant^2\max\{8,\lipcon\}}$ many iterations,
the extension of Theorem~\ref{ThmGodWell} to expectations guarantees
that the mean squared error is bounded as
\begin{align}
\label{EqnLSerrorUB}
\Exs \enorm{\Avgf{\tau} - \fstar}^2 & \leq C' \;
\frac{\critquant^2}{\strongcon^2},
\end{align}
where $C'$ is another constant depending on $\Radius$.  Here we have
used the fact that $\lipcon \geq \strongcon$ in simplifying the
expression.  It is worth noting that guarantee~\eqref{EqnLSerrorUB}
matches the best known upper bounds for kernel ridge regression
(KRR)---indeed, this must be the case, since a sharp analysis of KRR
is based on the same notion of localized Gaussian complexity (e.g. \cite{Bartlett02,Bar05})
.  Thus,
our results establish a strong parallel between the \emph{algorithmic
  regularization} of early stopping, and the \emph{penalized
  regularization} of kernel ridge regression.  Moreover, as will be
clarified in Section~\ref{SecLower}, under suitable regularity
conditions on the RKHS, the critical squared radius $\critquant^2$
also acts as a lower bound for the expected risk, meaning that our upper
bounds are not improvable in general.

Note that the critical radius $\critquant^2$ only depends on our
observations $\{(x_i,y_i)\}_{i=1}^\numobs$ through the solution of
inequality~\eqref{Eqn75perChocolate}.  In many cases, it is possible
to compute and/or upper bound this critical radius, so that a concrete
and valid stopping rule can indeed by calculated in advance.  In
Section~\ref{SecExamples}, we provide a number of settings in which
this can be done in terms of the eigenvalues
$\{\muhat_j\}_{j=1}^\numobs$ of the normalized kernel matrix.


\vspace*{0.2cm}

\subsubsection{Consequences for random design regression}

Thus far, our analysis has focused purely on the case of fixed design,
in which the sequence of covariates $\{x_i\}_{i=1}^\numobs$ is viewed
as fixed.  If we instead view the covariates as being sampled in an
i.i.d. manner from some distribution $\mprobx$ over $\Xspace$, then
the empirical error $\|\fhat - \fstar\|_\numobs^2 = \frac{1}{\numobs}
\sum_{i=1}^\numobs \big(f(x_i) - \fstar(x_i) \big)^2$ of a given
estimate $\fhat$ is a random quantity, and it is interesting to relate
it to the squared population $L^2(\mprobx)$-norm 
$\|\hatf - \fstar\|_2^2 = \Exs \big[ (\hatf(X) - \fstar(X))^2 \big]$.


In order to state an upper bound on this error, we introduce a
population analogue of the critical radius $\critquant$, which we
denote by $\popcritquant$.  Consider the set
\begin{align}
\widebar{\Ellipse}(\delta, 1) & \defn \Big \{f-g \, \mid \,
f,g\in\Hil, ~\|f - g\|_\Hil \leq 1, ~\|f - g\|_2 \leq \delta
\Big \}.
\end{align}
It is analogous to the previously defined set $\Eset{\delta}{1}$,
except that the empirical norm $\|\cdot\|_\numobs$ has been replaced
by the population version.  The population Gaussian complexity
localized at scale $\delta$ is given by
\begin{align}
\label{EqnGW}
\GWP \big(\widebar{\Ellipse}(\delta, 1) \big) & \defn \Exs_{w, X} \Big[
  \sup_{g \in \EllipseBar(\delta, 1)} \frac{1}{\numobs}
  \sum_{i=1}^\numobs w_i g(X_i) \Big],
\end{align}
where $\{w_i\}_{i=1}^\numobs$ are an i.i.d. sequence of standard
normal variates, and $\{X_i\}_{i=1}^\numobs$ is a second
i.i.d. sequence, independent of the normal variates, drawn according
to $\mprobx$.  Finally, the population critical radius
$\popcritquant$ is defined by equation~\eqref{EqnGW}, in which $\GW$
is replaced by $\GWP$.

\begin{cors}
  \label{CorRandDesign}
In addition to the conditions of Theorem~\ref{ThmGodWell}, suppose
that the sequence $\{(X_i, Y_i)\}_{i=1}^\numobs$ of covariate-response
pairs are drawn i.i.d. from some joint distribution $\PP$, and we
compute the boosting updates with step size 
\mbox{$\stsize \in (0, \min\{ \frac{1}{\lipcon}, \lipcon \}]$} and initialization $f^0 = 0$.  Then the
averaged function estimate $\Avgf{T}$ at time $T \defn \lfloor
\frac{1}{\critquant^2\max\{8,\lipcon\}} \rfloor$ satisfies the bound
\begin{align*}
\Exs_X \big(\Avgf{T}(X) - \fstar(X) \big)^2 \; = \; \ltwo{\Avgf{T} -
  \fstar}^2 \leq \tilde{c} \; \popcritquant^2
\end{align*}
with probability at least $1- c_1 \exp (-C_2\frac{m^2\numobs
  \critquant^2}{\level^2})$ over the random samples.
\end{cors}

The proof of Corollary~\ref{CorRandDesign} follows directly from
standard empirical process theory bounds~\cite{Bar05, RasWaiYu14} on
the difference between empirical risk $\enorm{\Avgf{T} - \fstar}^2$
and population risk $\ltwo{\Avgf{T} - \fstar}^2$. In particular, it
can be shown that $\ltwo{\cdot}$ and $\EmpNorm{\cdot}$ norms differ
only by a factor proportion to $\popcritquant$. Furthermore, one can
show that the empirical critical quantity $\critquant$ is bounded by
the population $\popcritquant$.  By combining both arguments the
corollary follows. We refer the reader to the
papers~\cite{Bar05,RasWaiYu14} for further details on such
equivalences.

It is worth comparing this guarantee with the past work of Raskutti et
al.~\cite{RasWaiYu14}, who analyzed the kernel boosting iterates of
the form~\eqref{EqnUpdate}, but with attention restricted to the special
case of the least-squares loss.  Their analysis was based on first decomposing
the squared error into bias and variance terms, then carefully
relating the combination of these terms to a particular bound on the
localized Gaussian complexity (see equation~\eqref{EqnMendelson}
below).  In contrast, our theory more directly analyzes the effective
function class that is explored by taking $T$ steps, so that the
localized Gaussian width~\eqref{EqnGW} appears more naturally.  In
addition, our analysis applies to a broader class of loss functions.

In the case of reproducing kernel Hilbert spaces, it is possible to
sandwich the localized Gaussian complexity by a function of the
eigenvalues of the kernel matrix.  Mendelson~\cite{Men02} provides
this argument in the case of the localized Rademacher complexity, but
similar arguments apply to the localized Gaussian complexity.  Letting
$\muhat_1 \geq \muhat_2 \geq \cdots \geq \muhat_\numobs \geq 0$ denote
the ordered eigenvalues of the normalized kernel matrix $\Kmat$,
define the function
\begin{align}
\label{EqnMendelson}
\KerR(\delta) & = \frac{1}{\sqrt{\numobs}} \sqrt{ \sum_{j=1}^\numobs
  \min \{ \delta^2, \muhat_j \}}.
\end{align}
Up to a universal constant, this function is an upper bound on the
Gaussian width $\GW \big(\Ellipse(\delta, 1) \big)$ for all $\delta \geq 0$, and up
to another universal constant, it is also a lower bound for all $\delta
\geq \frac{1}{\sqrt{\numobs}}$.
\subsection{Achieving minimax lower bounds}
\label{SecLower}


In this section, we show that the upper bound~\eqref{EqnLSerrorUB}
matches known minimax lower bounds on the error, so that our results
are unimprovable in general.  We establish this result for the class
of \emph{regular kernels}, as previously defined by Yang et
al.~\cite{YanPilWai17}, which includes the Gaussian and Sobolev
kernels as special cases.

The class of regular kernels is defined as follows. Let $\muhat_1 \geq
\muhat_2 \geq \cdots \geq \muhat_\numobs \geq 0$ denote the ordered
eigenvalues of the normalized kernel matrix $\Kmat$, and define the
quantity $\statdim \defn \argmin_{j=1,\dots,\numobs} \{\muhat_j \leq
\critquant^2\}$. A kernel is called \emph{regular} whenever there is a
universal constant $c$ such that the tail sum satisfies
$\sum_{j=\statdim+1}^\numobs \muhat_j \leq c \, \statdim
\critquant^2$. In words, the tail sum of the eigenvalues for regular
kernels is roughly on the same or smaller scale as the sum of the
eigenvalues bigger than $\critquant^2$.

For such kernels and under the Gaussian observation model ($Y_i \sim
N( \fstar(x_i), \sigma^2)$), Yang et al.~\cite{YanPilWai17} prove a
minimax lower bound involving $\critquant$.  In particular, they show
that the minimax risk over the unit ball of the Hilbert space is lower
bounded as
\begin{align}
\label{EqnYun}
\inf_{\fhat} \sup_{\HilNorm{\fstar} \leq 1} \Exs \enorm{\fhat -
  \fstar}^2 \geq c_\ell \critquant^2.
\end{align}
Comparing the lower bound~\eqref{EqnYun} with upper
bound~\eqref{EqnLSerrorUB} for our estimator $\Avgf{T}$ stopped after
$O(1/\critquant^2)$ many steps, it follows that the bounds proven in
Theorem~\ref{ThmGodWell} are unimprovable apart from constant factors.

We now state a generalization of this minimax lower bound, one which
applies to a sub-class of \emph{generalized linear models}, or GLM for
short.  In these models, the conditional distribution of the observed
vector $Y = (Y_1, \ldots, Y_\numobs)$ given $\big( \fstar(x_1),
\ldots, \fstar(x_\numobs) \big)$ takes the form
\begin{align}
  \label{EqnGLM}
  \Prob_\theta(y) = \prod_{i=1}^\numobs \Big[h(y_i)\exp\big(\frac{y_i
      \fstar(x_i) - \Phi(\fstar(x_i))}{s(\sigma)}\big) \Big],
\end{align}
where $s(\sigma)$ is a known scale factor and $\Phi : \real \to \real$
is the cumulant function of the generalized linear model.  As some concrete
examples:
\bcar
\item The linear Gaussian model is recovered by setting $s(\sigma) =
  \sigma^2$ and $\Phi(t) = t^2/2$.
\item The logistic model for binary responses $y \in \{-1, 1\}$ is
  recovered by setting $s(\sigma) = 1$ and $\Phi(t) = \log(1 +
  \exp(t))$.
\ecar

Our minimax lower bound applies to the class of GLMs for which the
cumulant function $\Phi$ is differentiable and has uniformly bounded
second derivative $|\Phi''| \leq L$.  This class includes the linear,
logistic, multinomial families, among others, but excludes (for
instance) the Poisson family.  Under this condition, we have the
following:
\begin{cors}
  \label{CorOptimal}
Suppose that we are given i.i.d. samples $\{y_i\}_{i=1}^{\numobs}$
from a GLM~\eqref{EqnGLM} for some function $\fstar$ in a regular
kernel class with $\|\fstar\|_\Hil \leq 1$.  Then running \mbox{$T
  \defn \lfloor \frac{1}{\critquant^2 \max\{8,\lipcon\}} \rfloor$}
iterations with step size \mbox{$\stsize \in (0,
  \min\{ \frac{1}{\lipcon}, \lipcon \}]$} and
$f^0 = 0$ yields an estimate $\Avgf{T}$ such that
  \begin{align}
    \label{EqnMinimaxGLM}
    \Exs \enorm{\Avgf{T} - \fstar}^2 & \, \asymp \, \inf_{\fhat}
    \sup_{\|\fstar\|_\Hil \leq 1} \Exs \enorm{\fhat - \fstar}^2.
  \end{align}
\end{cors}

Here $f(n) \asymp g(n)$ means $f(n) = c g(n)$ up to a universal constant $c \in (0,\infty)$.
As always, in the minimax claim~\eqref{EqnMinimaxGLM}, the infimum is
taken over all measurable functions of the input data and the
expectation is taken over the randomness of the response variables
$\{Y_i\}_{i=1}^\numobs$.  Since we know that 
\mbox{$\Exs \enorm{\Avgf{T} -
  \fstar}^2 \precsim \critquant^2$}, the way to prove
bound~\eqref{EqnMinimaxGLM} is by establishing 
\mbox{$\inf_{\fhat}
\sup_{\|\fstar\|_\Hil \leq 1} \Exs \enorm{\fhat - \fstar}^2 \succsim
\critquant^2$}.  See Section~\ref{SecProofOpt} for the proof of this
result.

At a high level, the statement in Corollary~\ref{CorOptimal} shows
that early stopping prevents us from overfitting to the data; in
particular, using the stopping time $T$ yields an estimate that
attains the optimal balance between bias and variance.

\section{Consequences for various kernel classes}
\label{SecExamples}

In this section, we apply Theorem~\ref{ThmGodWell} to derive some
concrete rates for different kernel spaces and then illustrate them
with some numerical experiments.  It is known that the complexity of
an RKHS in association with a distribution over the covariates
$\mprobx$ can be characterized by the decay rate~\eqref{EqnMercer} of
the eigenvalues of the kernel function.  In the finite sample setting,
the analogous quantities are the eigenvalues
$\{\muhat_j\}_{j=1}^\numobs$ of the normalized kernel matrix $\Kmat$.
The representation power of a kernel class is directly correlated with
the eigen-decay: the faster the decay, the smaller the function class.
When the covariates are drawn from the distribution $\mprobx$,
empirical process theory guarantees that the empirical and population
eigenvalues are close.


\subsection{Theoretical predictions as a function of decay}

In this section, let us consider two broad types of eigen-decay:
\begin{hclist}
\item {\bf{$\gamma$-exponential decay}}: For some $\gamma > 0$, the
  kernel matrix eigenvalues satisfy a decay condition of the form $\mu_j \leq
  c_1\exp(-c_2 j^\gamma)$, where $c_1,c_2$ are universal constants.
  Examples of kernels in this class include the Gaussian kernel,
  which for the Lebesgue measure satisfies such a bound with $\gamma = 2$ (real line) or $\gamma = 1$ (compact domain).
  %
\item {\bf{$\beta$-polynomial decay}}: For some $\beta > 1/2$, the
  kernel matrix eigenvalues satisfy a decay condition of the form $\mu_j \leq
  c_1 j^{-2 \beta}$, where $c_1$ is a universal constant.  Examples of
  kernels in this class include the $k^{th}$-order Sobolev spaces for
  some fixed integer $k \geq 1$ with Lebesgue measure on a bounded
  domain.  We consider Sobolev spaces that consist of functions that
  have $k^{th}$-order weak derivatives $f^{(k)}$ being Lebesgue
  integrable and $f(0) = f^{(1)}(0) = \dots = f^{(k-1)}(0)=0$.
  For such classes, the $\beta$-polynomial decay condition
  holds with $\beta = k$.
\end{hclist}

Given eigendecay conditions of these types, it is possible to compute
an upper bound on the critical radius $\critquant$. In particular,
using the fact that the function $\KerR$ from
equation~\eqref{EqnMendelson} is an upper bound on the function $\GW
\big(\Ellipse(\delta, 1) \big)$,
we can show that for $\gamma$-exponentially decaying kernels, we have
$\critquant^2 \precsim \frac{(\log \numobs)^{1/\gamma}}{\numobs}$,
whereas for $\beta$-polynomial kernels, we have $\critquant^2 \precsim
\numobs^{- \frac{2 \beta}{2 \beta +1}}$ up to universal constants.
Combining with our Theorem~\ref{ThmGodWell}, we obtain the following
result:

\vspace*{0.5cm}

\begin{cors}[Bounds based on eigendecay]
  \label{CorClass}
Under the conditions of Theorem~\ref{ThmGodWell}:
\begin{enumerate}
\item[(a)] For kernels with $\gamma$-exponential eigen-decay, we have 
\begin{subequations}
\begin{align}
\Exs \enorm{\Avgf{T} - \fstar}^2 \leq c \, \frac{\log^{1/\gamma}
  \numobs}{ \numobs} ~~ \mbox{at $ T \asymp
  \frac{\numobs}{\log^{1/\gamma} \numobs}$ steps.}
  \end{align}
\item[(b)] For kernels with $\beta$-polynomial eigen-decay, we have
\begin{align}
\Exs \enorm{\Avgf{T} - \fstar}^2 \leq c \, \numobs^{-2\beta/(2\beta+1)}
~~ \mbox{at $T \asymp \numobs^{2\beta/(2\beta+1)}$ steps.}
\end{align}
\end{subequations}
\end{enumerate}
\end{cors}
\noindent See Section~\ref{SecProofClass} for the proof of
Corollary~\ref{CorClass}.

In particular, these bounds hold for LogitBoost and AdaBoost.  We note
that similar bounds can also be derived with regard to risk in
$L^2(\PP_\numobs)$ norm as well as the excess risk $\Loss(f^T) -
\Loss(f^*)$.

To the best of our knowledge, this result is the first to show
non-asymptotic and optimal statistical rates for the
$\|\cdot\|_n^2$-error when early stopping LogitBoost or AdaBoost
with an explicit dependence of the stopping rule on $\numobs$.  Our
results also yield similar guarantees for $L^2$-boosting, as has been
established in past work~\cite{RasWaiYu14}.  Note that we can observe
a similar trade-off between computational efficiency and statistical
accuracy as in the case of kernel least-squares
regression~\cite{Yao07, RasWaiYu14}: although larger kernel classes
(e.g. Sobolev classes) yield higher estimation errors, boosting
updates reach the optimum faster than for a smaller kernel class
(e.g. Gaussian kernels).

\subsection{Numerical experiments} 
\label{SecSimulations}

We now describe some numerical experiments that provide illustrative
confirmations of our theoretical predictions.  While we have applied
our methods to various kernel classes, in this section, we present
numerical results for the first-order Sobolev kernel as two typical
examples for exponential and polynomial eigen-decay kernel classes.

Let us start with the first-order Sobolev space of Lipschitz functions
on the unit interval $[0,1]$, defined by the kernel
\mbox{$\KerFun(x,x') = 1 + \min(x,x')$,} and with the design points
$\{x_i\}_{i=1}^n$ set equidistantly over $[0,1]$. Note that the
equidistant design yields $\beta$-polynomial decay of the eigenvalues of
$\Kmat$ with $\beta=1$ as in the case when $x_i$ are drawn i.i.d. from the uniform measure
on $[0,1]$. 
Consequently we have that $\critquant^2 \asymp
\numobs^{-2/3}$.  Accordingly, our theory predicts that the stopping
time $T = (c \numobs)^{2/3}$ should lead to an estimate $\Avgf{T}$
such that $\|\Avgf{T} - \fstar\|_\numobs^2 \precsim \numobs^{-2/3}$.

In our experiments for $L^2$-Boost, we sampled $Y_i$ according to $Y_i
= \fstar(x_i) + w_i$ with $w_i \sim \NORMAL(0,0.5)$, which corresponds
to the probability distribution $\PP(Y \mid x_i) =
\NORMAL(\fstar(x_i); 0.5)$, where $\fstar(x) = |x - \frac{1}{2}| -
\frac{1}{4}$ is defined on the unit interval $[0,1]$. By construction,
the function $\fstar$ belongs to the first-order Sobolev space with
$\|\fstar\|_\Hil = 1$.  For LogitBoost, we sampled $Y_i$ according to
$\text{Bin}(p(x_i), 5)$ where $p(x) =
\frac{\exp(\fstar(x))}{1+\exp(\fstar(x))}$.  In all cases, we fixed
the initialization $f^0 = 0$, and ran the updates~\eqref{EqnUpdate}
for $L^2$-Boost and LogitBoost with the constant step size $\alpha =
0.75$.  We compared various stopping rules to the \emph{oracle gold
  standard} $G$, meaning the procedure that examines all iterates
$\{f^t\}$, and chooses the stopping time $G = \arg \min_{t \geq 1}
\|f^t - \fstar\|_\numobs^2$ that yields the minimum prediction error.
Of course, this procedure is unimplementable in practice, but it
serves as a convenient lower bound with which to compare.

\begin{figure}[t]
  \begin{center}
    \begin{tabular}{ccc}
      \widgraph{0.47\textwidth}{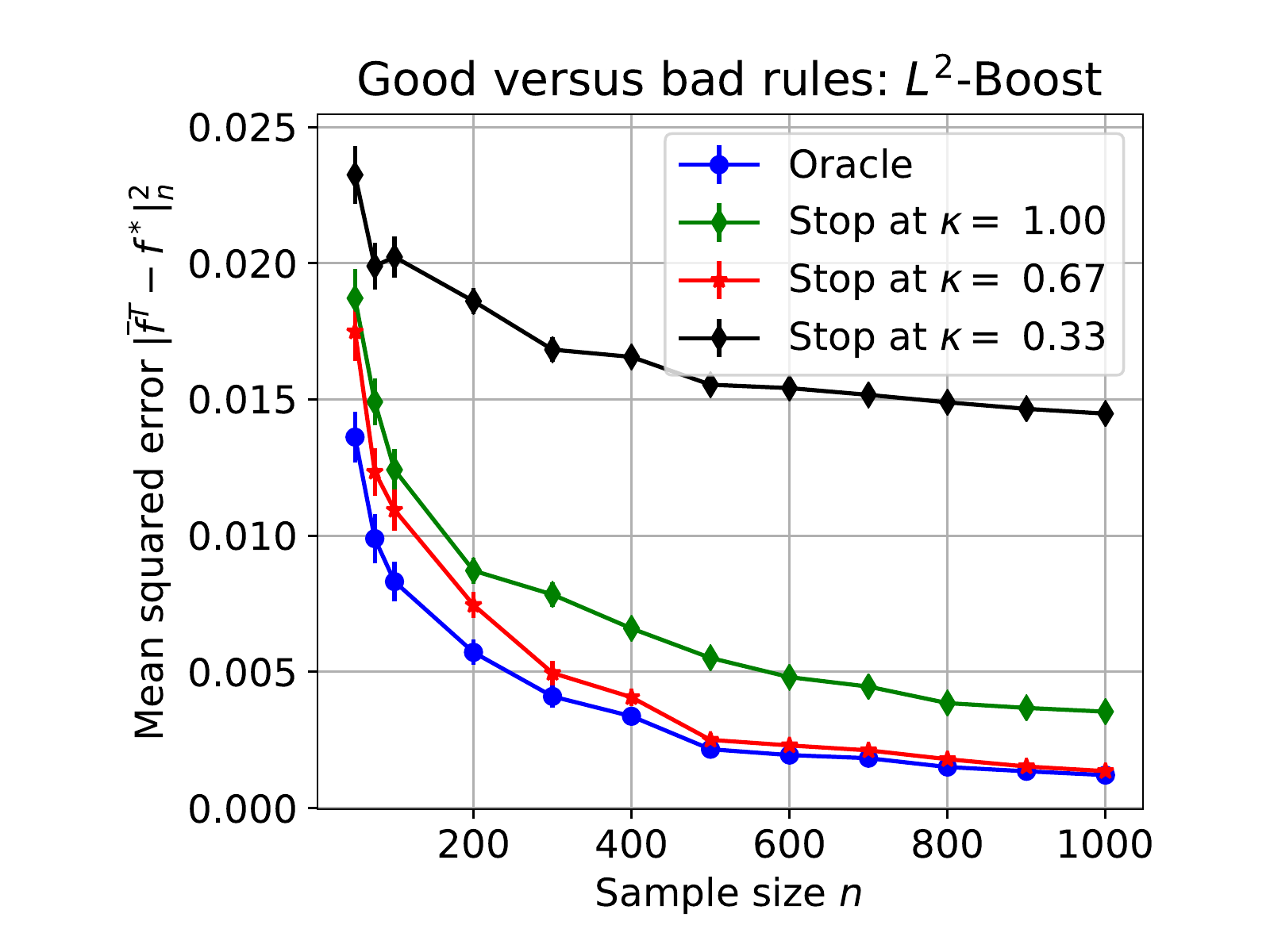} &&
      \widgraph{0.47\textwidth}{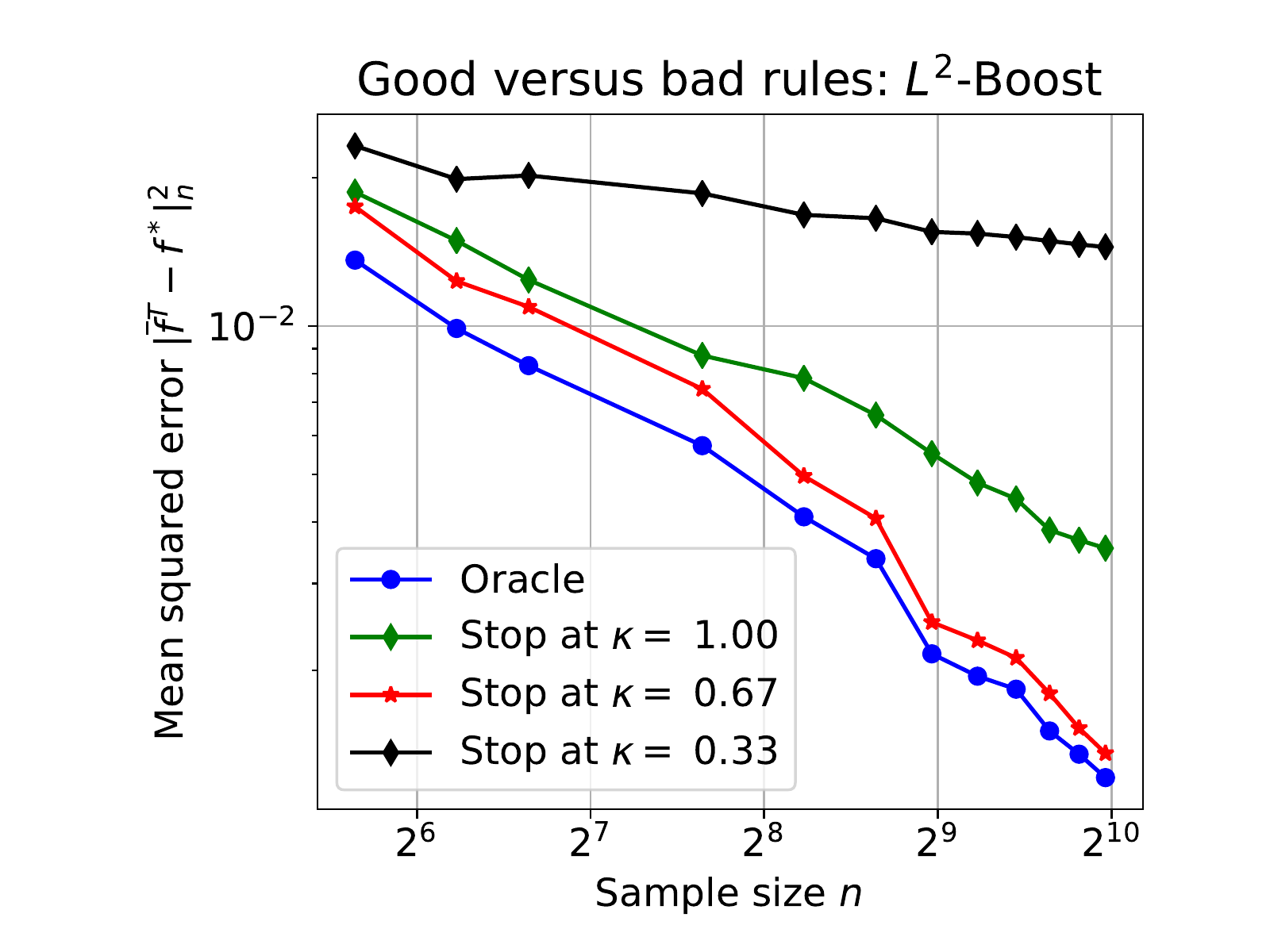} \\
      (a) \vspace{-0.05in}& & (b) \vspace{-0.05in}
    \end{tabular}
  \end{center}
  \caption{The mean-squared errors for the stopped iterates $\bar{f}^T$ at
    the Gold standard, i.e. iterate with the minimum error among all
    unstopped updates (blue) and at $T = (7 \numobs)^{\kappa}$ (with the
    theoretically optimal  $\kappa = 0.67$  in red,
    $\kappa = 0.33$ in black and $\kappa = 1$ in green) for (a) $L^2$-Boost
    and (b) LogitBoost.}
  \vspace{-0.15in}
  \label{FigError}
\end{figure}


Figure~\ref{FigError} shows plots of the mean-squared error $\|\bar{f}^T -
\fstar\|_\numobs^2$ over the sample size $\numobs$  averaged over $40$ trials, 
for the gold standard $T = G$ and
stopping rules based on $T = (7 \numobs)^{\kappa}$ for different choices of
$\kappa$.
Error bars correspond to the standard errors computed
from our simulations. Panel (a) shows the behavior for $L^2$-boosting,
whereas panel (b) shows the behavior for LogitBoost.  

Note that both plots are qualitatively similar and that the
theoretically derived stopping rule $T = (7 \numobs)^{\kappa}$ with
$\kappa^* = 2/3 = 0.67$, while slightly worse than the Gold
standard, tracks its performance closely.
We also performed simulations for some ``bad'' stopping rules, in
particular for an exponent $\kappa$ \emph{not equal} to
$\kappa^* = 2/3$, indicated by the green and black curves.
In the log scale plots in Figure~\ref{FigBad} we can clearly see that
for $\kappa \in \{0.33, 1\}$ the performance is indeed much worse, with the
difference in slope even suggesting a different scaling of the error
with the number of observations $\numobs$. 
Recalling our discussion
for Figure~\ref{FigEarly}, this phenomenon likely occurs due to
underfitting and overfitting effects.
These qualitative shifts are consistent with our theory.


\begin{figure}[t]
  \begin{center}
    \begin{tabular}{ccc}
      \widgraph{0.47\textwidth}{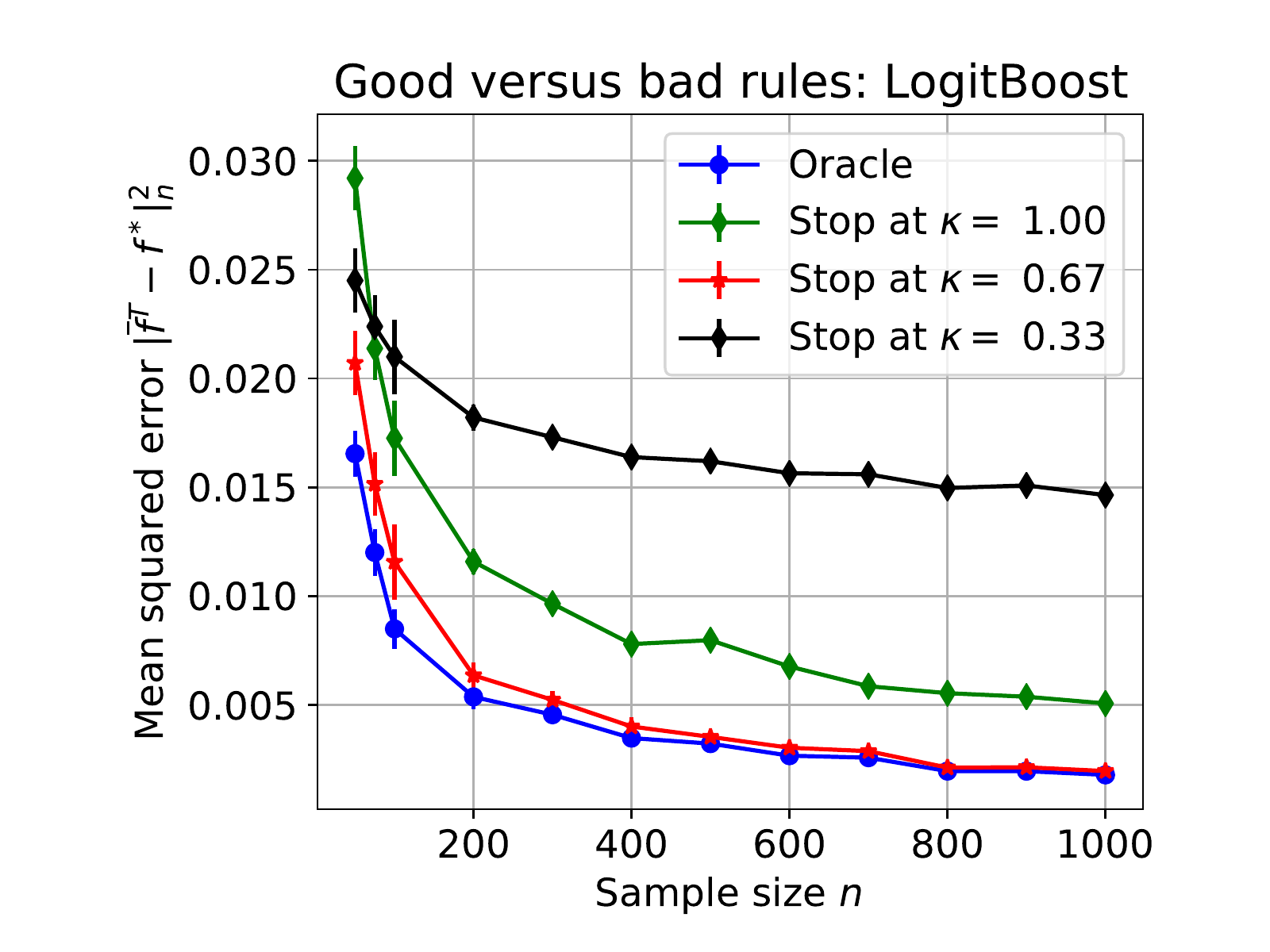} &&
      \widgraph{0.47\textwidth}{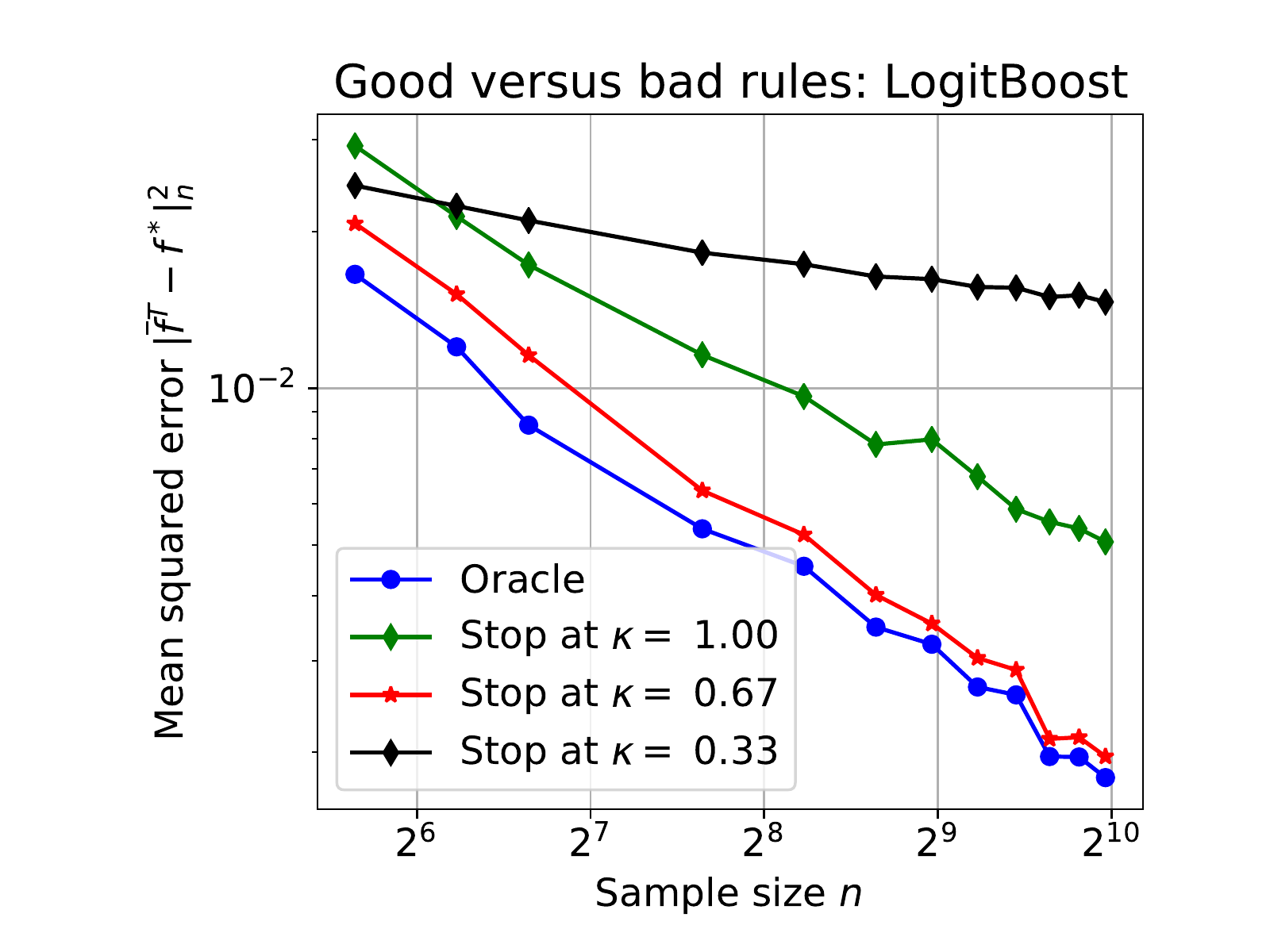} \\
      (a) \vspace{-0.05in} & & (b) \vspace{-0.05in}
    \end{tabular}
  \end{center}
  \caption{Logarithmic plots of the mean-squared errors at the Gold
    standard in blue and at $T = (7 \numobs)^\kappa$ (with the theoretically optimal rule for $\kappa = 0.67$ in red, $\kappa = 0.33$ in black and $\kappa = 1$
    in green) for (a) $L^2$-Boost and (b) LogitBoost.
    }
  \label{FigBad}
\end{figure}


\vspace*{1cm}

\section{Proof of main results}
\label{SecProofs}

In this section, we present the proofs of our main results. The
technical details are deferred to \autoref{SecTechnicalProofs}.

In the following, recalling the discussion in
Section~\ref{SecBoostKer}, we denote the vector of function values
of a function  $f \in
\Hil$ evaluated at $(x_1,x_2,\ldots,x_\numobs)$ as $\fval_f \defn
f(x_1^\numobs) = (f(x_1),f(x_2),\ldots f(x_\numobs)) \in
\real^{\numobs}$, where we omit the subscript $f$ when it is clear
from the context. As mentioned in the main text, updates on the
function value vectors $\fval^t \in \real^\numobs$ correspond uniquely
to updates of the functions $f^t \in \Hil$. In the following we
repeatedly abuse notation by defining the Hilbert norm and empirical
norm on vectors in $\Delta \in \range(K)$ as
\begin{align*}
\HilNorm{\Delta}^2 = \frac{1}{\numobs} \Delta^T
\Kmat^{\pseudinv}\Delta \quad\text{ and }\quad \EmpNorm{\Delta}^2 =
\frac{1}{\numobs} \ltwo{\Delta}^2,
\end{align*}
where $\Kmat^{\pseudinv}$ is the pseudoinverse of $\Kmat$. 
We also use $\HilNormball(\theta, r)$ to denote the ball
with respect to the $\HilNorm{\cdot}$-norm in $\range(K)$.


\subsection{Proof of Theorem~\ref{ThmGodWell}}
\label{SecProofThm}

The proof of our main theorem is based on a sequence of lemmas, all of
which are stated with the assumptions of Theorem~\ref{ThmGodWell} in
force.  The first lemma establishes a bound on the empirical norm
$\enorm{\cdot}$ of the error \mbox{$\DelIt{t+1} \defn \fvalit{t+1} -
  \fvalstar$,} provided that its Hilbert norm is suitably controlled.

\begin{lems}
\label{LemMasterBound}
For any stepsize $\stsize \in (0,~\frac{1}{\lipcon}]$ and any
  iteration $t$ 
  we have
\begin{align}
  \label{EqnMasterBound}
  \notag \frac{\strongcon}{2 } \enorm{\DelIt{t+1}}^2 &\leq
  \frac{1}{2\stsize} \Big \{ \HilNorm{\DelIt{t}}^2 -
  \HilNorm{\DelIt{t+1}}^2 \Big \} + \inprod{\nabla \Loss(\fvalstar +
    \DelIt{t}) - \nabla \EmpLoss(\fvalstar + \DelIt{t})}{\DelIt{t+1}}.
\end{align}
\end{lems}
\noindent See Section~\ref{SecLemMasProof} for the proof of this claim.

The second term on the right-hand side of the bound~\eqref{EqnMasterBound}
involves the difference between the population and empirical gradient
operators.  Since this difference is being evaluated at the random
points $\DelIt{t}$ and $\DelIt{t+1}$, the following lemma establishes
a form of uniform control on this term.

Let us define the set
\begin{align}
  \HANASET & \defn \Biggr \{ \Delta, \DelTil \in \real^\numobs \mid
  \HilNorm{\Delta} \geq 1,  ~\mbox{and}~
  \Delta, ~\DelTil \in \HilNormball(0,2\Radius)
  \Biggr \},
\end{align}
and consider the uniform bound
\begin{align}
  \label{EqnPeeling}
  \notag \inprod{\nabla \Loss(\fvalstar + \DelTil) & - \nabla
    \EmpLoss(\fvalstar + \DelTil)}{\Delta} \leq 2\delta_{\numobs} \enorm{\Delta} \\
    &+ 2\delta_{\numobs}^2 \HilNorm{\Delta} 
    +
  \frac{\strongcon}{ \AnnCon} \enorm{\Delta}^2 \quad \mbox{for all
    $\Delta, \DelTil \in \HANASET$.}
\end{align}

\begin{lems}
\label{LemPeeling}
Let $\Event$ be the event that bound~\eqref{EqnPeeling}
holds. There are universal constants $(c_1, c_2)$ such that
$\mprob[\Event] \geq 1 - c_1 \exp(-c_2 \frac{\strongcon^2 \numobs
  \delta_\numobs^2}{\level^2})$.
\end{lems}
\noindent See Section~\ref{SecProofPeeling} for the proof of
Lemma~\ref{LemPeeling}.

\vspace*{0.2cm}

Note that Lemma~\ref{LemMasterBound} applies only to error iterates
with a bounded Hilbert norm.  Our last lemma provides this control for
some number of iterations:
\begin{lems}
\label{LemHilBound}
There are constants $(C_1, C_2)$ independent of $\numobs$ such that for any step size
$\stsize \in \big(0, \min\{\lipcon, \frac{1}{\lipcon}\} \big]$, we have
\begin{align}
\label{EqnDeltaStayBounded}
\HilNorm{\DelIt{t}} \leq \Radius \qquad \mbox{for all iterations $t
  \leq \frac{\strongcon}{8 \lipcon \critquant^2}$}
\end{align}
with probability at least
$1 - C_1 \exp(- C_2 \numobs \delta_\numobs^2)$, where $C_2 = \max\{ \frac{\strongcon^2
 }{\level^2}, 1\}$.
\end{lems}
\noindent See Section~\ref{SecLemHilProof} for the proof of this lemma which also uses Lemma~\ref{LemPeeling}.

\vspace*{.1in}

Taking these lemmas as given, we now complete the proof of the
theorem.  We first condition on the event $\Event$ from
Lemma~\ref{LemPeeling}, so that we may apply the
bound~\eqref{EqnPeeling}.  We then fix some iterate $t$ such that $t <
\frac{\strongcon}{8 \lipcon \critquant^2} - 1$, and condition on the event that the
bound~\eqref{EqnDeltaStayBounded} in Lemma~\ref{LemHilBound} holds, so
that we are guaranteed that $\HilNorm{\DelIt{t+1}} \leq \Radius$.  We
then split the analysis into two cases:

\paragraph{Case 1}  First, suppose that
\mbox{$\enorm{\DelIt{t+1}} \leq \critquant \Radius$}.  In this case,
inequality~\eqref{EqnEnormBound} holds directly.

\paragraph{Case 2} Otherwise, we may assume that
$\enorm{\DelIt{t+1}} > \critquant \HilNorm{\DelIt{t+1}}$.  Applying
the bound~\eqref{EqnPeeling} with the choice $(\DelTil, \Delta) =
(\DelIt{t}, \DelIt{t+1})$ yields
\begin{align}
\label{EqnGrizzly}
\inprod{\nabla \Loss(\fvalstar + \DelIt{t}) &- \nabla
  \EmpLoss(\fvalstar + \DelIt{t})}{\DelIt{t+1}} \leq 4 \critquant
\enorm{\DelIt{t+1}} + \frac{\strongcon}{ \AnnCon}\enorm{\DelIt{t+1}}^2.
\end{align}
Substituting inequality~\eqref{EqnGrizzly} back into
equation~\eqref{EqnMasterBound} yields
\begin{align*}
  \frac{\strongcon}{2 } \enorm{\DelIt{t+1}}^2 \leq
  \frac{1}{2\stsize} \Big \{ \HilNorm{\DelIt{t}}^2 & -
  \HilNorm{\DelIt{t+1}}^2 \Big \} 
  + 4 \critquant \enorm{\DelIt{t+1}} +
  \frac{\strongcon}{\AnnCon}\enorm{\DelIt{t+1}}^2.
\end{align*}

Re-arranging terms 
yields the bound
\begin{align}
\label{EqnSquared}
   \sceff \strongcon \enorm{\DelIt{t+1}}^2 \leq \Hdiff{t} + 4 \delta_n
  \enorm{\DelIt{t+1}},
\end{align}
where we have introduced the shorthand notation $\Hdiff{t} \defn
\frac{1}{2 \stsize} \Big \{ \HilNorm{\DelIt{t}}^2 -
\HilNorm{\DelIt{t+1}}^2 \Big \}$, as well as $\sceff = \frac{1}{2} -
\frac{1}{\AnnCon}$ 

Equation~\eqref{EqnSquared} defines a quadratic inequality with
respect to $\enorm{\DelIt{t+1}}$; solving it and making use of the
inequality $(a+b)^2 \leq 2a^2 + 2 b^2$ yields the bound
\begin{align}
  \label{EqnSquaredTwo}
\enorm{\DelIt{t+1}}^2 \leq \frac{c \delta_n^2}{\sceff^2\strongcon^2} +
\frac{2 \Hdiff{t} }{\sceff\strongcon},
\end{align}
for some universal constant $c$.  By telescoping
inequality~\eqref{EqnSquaredTwo}, we find that
\begin{align}
\label{EqnNearFinal}
\frac{1}{T} \sum_{t=1}^T \enorm{\DelIt{t}}^2 &\leq \frac{c
  \delta_n^2 }{\sceff^2\strongcon^2} + \frac{1}{T} \sum_{t=1}^T \frac{2
  \Hdiff{t} }{\sceff \strongcon}\\
  &\leq \frac{c
  \delta_n^2 }{\sceff^2\strongcon^2} + \frac{1}{\stsize
  \sceff\strongcon T} [\HilNorm{\DelIt{0}}^2 - \HilNorm{\DelIt{T}}^2].
\end{align}
By Jensen's inequality, we have
\begin{align*}
  \enorm{\Avgf{T} - \fstar}^2 & = \enorm{\frac{1}{T} \sum_{t=1}^T
    \DelIt{t}}^2 \; \leq \; \frac{1}{T} \sum_{t=1}^T
  \enorm{\DelIt{t}}^2,
\end{align*}
so that inequality~\eqref{EqnEnormBound} follows from the
bound~\eqref{EqnNearFinal}.

On the other hand, by the smoothness assumption, we have
\begin{align*}
  \Loss(\Avgf{T}) - \Loss(\fstar) \leq \frac{\lipcon}{2}
  \enorm{\Avgf{T} - \fstar}^2,
\end{align*}
from which inequality~\eqref{EqnMasterUnwrapped} follows.


\subsection{Proof of Corollary~\ref{CorOptimal}}
\label{SecProofOpt}

Similar to the proof of Theorem 1 in Yang et al.~\cite{YanPilWai17}, a
generalization can be shown using a standard argument of Fano’s
inequality.  By definition of the transformed parameter
$\theta = D U \alpha$ with $K = U^T D U$, we have for any
estimator $\fhat  = \sqrt{n} U^T\theta$ that
$\enorm{\fhat - \fstar}^2 = \ltwo{\theta - \thetastar}^2$. Therefore
our goal is to lower bound the Euclidean error
$\ltwo{\theta -\thetastar}$ of any estimator of $\thetastar.$
Borrowing Lemma 4 in Yang et al.~\cite{YanPilWai17}, there exists
$\delta/2$-packing of the set
$B = \{\theta \in \real^\numobs \mid \ltwo{D^{-1/2}\theta} \leq 1\}$
of cardinality $M = e^{d_n/64}$ with
$d_n \defn \arg \min_{j=1,\dots, n} \{\mu_j \leq \delta_n^2\}.$ 
This is done through packing the following subset of $B$ 
\begin{align*}
  \Ellipse(\delta) \defn \Big\{ \theta\in \real^{\numobs} ~\mid~ \sum_{j=1}^{\numobs} \frac{\theta_j^2}{\min\{\delta^2,\mu_j\}} \leq 1 \Big\}.
\end{align*}
Let us denote the packing set by $\{\theta^1,\ldots,\theta^M\}.$ Since $\theta \in \Ellipse(\delta)$, by simple calculation, we have $\ltwo{\theta^i} \leq \delta$.

By considering the random ensemble of regression problem in which 
we first draw at index $Z$ at random from the index set $[M]$
and then condition on $Z = z$, we observe $\numobs$ i.i.d samples
$y_1^n := \{y_1,\dots,y_n\}$
from $\Prob_{\theta^z}$, 
Fano's inequality implies that 
\begin{align*}
  \Prob(\ltwo{\thetahat - \thetastar} \geq \frac{\delta^2}{4})
  \geq 1 - \frac{I(y_1^\numobs; Z) + \log 2}{\log M}.
\end{align*}
where $I(y_{1}^\numobs ;Z)$ is the mutual information between the samples $Y$ and the random index $Z$.

So it is only left for us to control the mutual information 
$I(y_{1}^\numobs ;Z)$. 
Using the mixture representation, $\bar{\Prob} = \frac{1}{M}
\sum_{i=1}^M \Prob_{\theta^i}$ and the convexity of the Kullback–Leibler divergence, we have 
\begin{align*}
  I(y_1^\numobs;Z) = \frac{1}{M} \sum_{j=1}^M\divKL{\Prob_{\theta^j}}{\bar{\Prob}}
  \leq \frac{1}{M^2} \sum_{i,j} \divKL{\Prob_{\theta^i}}{\Prob_{\theta^j}}.
\end{align*}
We now claim that 
\begin{align}
\label{EqnKLBound}
  \divKL{\Prob_\theta(y)}{\Prob_{\theta'}(y)} \leq \frac{n L\ltwo{\theta - \theta'}^2}{s(\sigma)}.
\end{align}
Since each $\ltwo{\theta^i}\leq \delta$, triangle inequality yields
$\ltwo{\theta_i - \theta_j} \leq 2\delta$ for all $i\neq j$.  It is
therefore guaranteed that
 \begin{align*}
  I(y_1^\numobs;Z) \leq \frac{4n L \delta^2}{s(\sigma)}.
\end{align*}
Therefore, similar to Yang et al.~\cite{YanPilWai17}, following by the
fact that the kernel is regular and hence
$s(\sigma) d_n \geq c n \critquant^2$, any estimator $\fhat$ has
prediction error lower bounded as
\begin{align*}
  \sup_{\|\fstar\|_{\Hil}\leq 1} \Exs \enorm{\fhat - \fstar}^2 
  \geq c_l \delta_n^2.
\end{align*}
Corollary~\ref{CorOptimal} thus follows using the upper bound in
Theorem~\ref{ThmGodWell}.

\paragraph*{Proof of inequality~\eqref{EqnKLBound}}

Direct calculations of the KL-divergence yield
\begin{align}
\label{EqnKL-step1}
\notag \divKL{\Prob_\theta(y)}{\Prob_{\theta'}(y)} 
= &\int
\log(\frac{\Prob_\theta(y)}{\Prob_{\theta'}(y)})\Prob_\theta(y) dy\\
%
%
\notag =& \frac{1}{s(\sigma)}\sum_{i=1}^\numobs \Phi(\sqrt{\numobs}\inprod{u_i}{\theta'})  - \Phi(\sqrt{\numobs}\inprod{u_i}{\theta}) \\
& + \frac{\sqrt{\numobs}}{s(\sigma)} \int \sum_{i=1}^\numobs \big[ y_i\inprod{u_i}{\theta-\theta'}
\big] \Prob_\theta dy.
\end{align}
To further control the right hand side of
expression~\eqref{EqnKL-step1}, we concentrate on expressing $\int
\sum_{i=1}^n y_i u_i \Prob_{\theta} d y$ differently.  Leibniz's rule
allow us to inter-change the order of integral and derivative, so that
\begin{align}
  \label{eq:Leibniz}
  \int \frac{d P_\theta}{d \theta} dy 
  =  \frac{d}{d \theta} \int P_\theta dy  = 0.
\end{align}
Observe that 
\begin{align*}
  \int \frac{d P_\theta}{d \theta} dy 
  = \frac{\sqrt{\numobs}}{s(\sigma)} \int P_\theta \cdot \sum_{i=1}^{\numobs} 
  u_i \big(y_i - \Phi'(\sqrt{\numobs}\inprod{u_i}{\theta'}) \big) dy
\end{align*}
so that equality \eqref{eq:Leibniz} yields
\begin{align*}
  \int \sum_{i=1}^{\numobs} y_iu_i \Prob_\theta dy = 
  \sum_{i=1}^{\numobs}  u_i\Phi'(\sqrt{\numobs}\inprod{u_i}{\theta}).
\end{align*}
Combining the above inequality with expression~\eqref{EqnKL-step1},
the KL divergence between two generalized linear models
$\Prob_\theta, \Prob_{\theta'}$ can thus be written as
\begin{align}
\label{EqnKL-step2}
\notag \divKL{\Prob_\theta(y)}{\Prob_{\theta'}(y)} 
= \frac{1}{s(\sigma)}
\sum_{i=1}^\numobs \Phi(\sqrt{\numobs}\inprod{u_i}{\theta'}) &- \Phi(\sqrt{\numobs}\inprod{u_i}{\theta}) \\
-  \sqrt{\numobs} \inprod{u_i}{\theta' - \theta} \Phi'(\sqrt{\numobs}\inprod{u_i}{\theta}).
\end{align}
Together with the fact that 
\begin{align*}
  |\Phi(\sqrt{\numobs}\inprod{u_i}{\theta'}) - \Phi(\sqrt{\numobs}\inprod{u_i}{\theta}) - \sqrt{\numobs} \inprod{u_i}{\theta' - \theta} \Phi'(\sqrt{\numobs}\inprod{u_i}{\theta})| \\
  \leq \numobs L\ltwo{\theta - \theta'}^2.
\end{align*}
which follows by assumption on $\Phi$ having a uniformly bounded
second derivative.
Putting the above inequality with inequality~\eqref{EqnKL-step2} establishes our claim~\eqref{EqnKLBound}.

\subsection{Proof of Corollary~\ref{CorClass}}
\label{SecProofClass}

The general statement follows directly from Theorem~\ref{ThmGodWell}.
In order to invoke Theorem~\ref{ThmGodWell} for the particular cases
of LogitBoost and AdaBoost, we need to verify the conditions,
i.e. that the \mMC~ and $\phi'$-boundedness conditions hold for the
respective loss function over the ball
$\HilNormBall(\fvalstar, 2\Radius)$.  The following lemma provides
such a guarantee:
\begin{lems}
  \label{LemLogAda}
With $\Diam \defn \Radius + \HilNorm{\fvalstar}$, the logistic
regression cost function satisfies the \mMC~with parameters
\begin{align*}
\strongcon = \frac{1}{e^{-D} + e^D + 2}, \quad \lipcon = \frac{1}{4},
\quad \mbox{and} \quad \PhiB = 1.
\end{align*}
The AdaBoost cost function satisfies the \mMC~with parameters
\begin{align*}
\strongcon = \E^{-\Diam}, \quad \lipcon = \E^{\Diam}, \quad \mbox{and}
\quad \PhiB = \E^{\Diam}.
\end{align*}
\end{lems}
\noindent 
See Section~\ref{SecProofLogAda} for the proof of
Lemma~\ref{LemLogAda}.

\paragraph{$\gamma$-exponential decay} 
If the kernel eigenvalues satisfy a decay condition of the form $\mu_j
\leq c_1\exp(-c_2 j^\gamma)$, where $c_1,c_2$ are universal constants,
the
function $\KerR$ from equation~\eqref{EqnMendelson} can be upper
bounded as
\begin{align*}
\KerR(\delta) = \sqrt{\frac{2}{\numobs}}\sqrt{\sum_{i=1}^{\numobs}
  \min\{\delta^2, \mu_j \}} \leq
\sqrt{\frac{2}{\numobs}}\sqrt{k\delta^2 + \sum_{j=k+1}^{\numobs} c_1
  e^{- c_2 j^2}},
\end{align*}
where $k$ is the smallest integer such that $c_1 \exp(- c_2 k^\gamma)
< \delta^2$.  Since the localized Gaussian width $\GW
\big(\Ellipse_\numobs(\delta, 1) \big)$ can be sandwiched above and below by
multiples of $\KerR(\delta)$, some algebra shows that the critical
radius scales as $\critquant^2 \asymp
\frac{\numobs}{\log(\numobs)^{1/\gamma}\sigma^2}$.

Consequently, if we take $T \asymp
\frac{\log(\numobs)^{1/\gamma}\sigma^2}{\numobs}$ steps, then
Theorem~\ref{ThmGodWell} guarantees that the averaged estimator
$\AvgEst^T$ satisfies the bound
\begin{align*}
  \enorm{\AvgEst^T - \fvalstar}^2 ~\lesssim~
  \left(\frac{1}{\stsize \strongcon} + 
  \frac{1}{\strongcon^2} \right)
  \frac{\log^{1/\gamma} \numobs}{\numobs}\level^2 ,
\end{align*}
with probability $1- c_1 \Exp (- c_2\strongcon^2
\log^{1/\gamma}\numobs)$.

\paragraph{$\beta$-polynomial decay}
Now suppose that the kernel eigenvalues satisfy a decay condition of
the form $\mu_j \leq c_1 j^{-2 \beta}$ for some $\beta > 1/2$ and
constant $c_1$.  In this case, a direct calculation yields the bound
\begin{align*}
\KerR(\delta) \leq \sqrt{\frac{2}{\numobs}}\sqrt{k\delta^2 +
  c_2\sum_{j=k+1}^{\numobs} j^{-2}},
\end{align*}
where $k$ is the smallest integer such that $c_2 k^{-2} < \delta^2$.
Combined with upper bound $c_2 \sum_{j=k+1}^{\numobs} j^{-2} \leq c_2
\int_{k+1} j^{-2} \leq k\delta^2$, we find that the critical radius
scales as $\critquant^2 \asymp \numobs^{- 2\beta/(1+2\beta)}$.

Consequently, if we take $T \asymp \numobs^{- 2 \beta/(1+2\beta)}$
many steps, then Theorem~\ref{ThmGodWell} guarantees that the averaged
estimator $\AvgEst^T$ satisfies the bound
\begin{align*}
  \enorm{\AvgEst^T - \fvalstar}^2 ~\leq ~ \left(\frac{1}{\stsize
    \strongcon} + \frac{1}{\strongcon^2} \right)
  \left(\frac{\level^2}{\numobs}\right)^{2\beta/(2\beta + 1)},
\end{align*}
with probability at least $1- c_1 \Exp (- c_2
\strongcon^2(\frac{\numobs}{\level^2})^{1/(2\beta+1)})$.


\section{Discussion}

In this paper, we have proven non-asymptotic bounds for early stopping
of kernel boosting for a relatively broad class of loss functions.
These bounds allowed us to propose simple stopping rules which, for
the class of regular kernel functions~\cite{YanPilWai17}, yield
minimax optimal rates of estimation. Although the connection between
early stopping and regularization has long been studied and explored
in the theoretical literature and applications alike, to the best of
our knowledge, this paper is the first one to establish a general
relationship between the statistical optimality of stopped iterates
and the localized Gaussian complexity.  This connection is important,
because this localized Gaussian complexity measure, as well as its Rademacher
analogue, are now well-understood to play a central role in
controlling the behavior of estimators based on
regularization~\cite{vandeGeer00,Bar05,Kolt06,Wai17}.

There are various open questions suggested by our results. The
stopping rules in this paper depend on the eigenvalues of the
empirical kernel matrix; for this reason, they are data-dependent and
computable given the data.  However, in practice, it would be
desirable to avoid the cost of computing all the empirical
eigenvalues.  Can fast approximation techniques for kernels be used to
approximately compute our optimal stopping rules?  Second, our current
theoretical results apply to the averaged estimator $\bar{f}^T$.  We
strongly suspect that the same results apply to the stopped estimator
$f^T$, but some new ingredients are required to extend our proofs.

\subsection*{Acknowledgments}
This work was partially supported by DOD Advanced Research Projects
Agency W911NF-16-1-0552, National Science Foundation grant
NSF-DMS-1612948, and Office of Naval Research Grant DOD-ONR-N00014.


\vspace*{1cm}


\vspace*{1cm}

\appendix


\section{Proof of technical lemmas}
\label{SecTechnicalProofs}

\subsection{Proof of Lemma~\ref{LemMasterBound}}
\label{SecLemMasProof}

Recalling that $\Kmat^{\pseudinv}$ denotes the pseudoinverse of
$\Kmat$, our proof is based on the linear transformation
\begin{align*}
  \zval \defn \numobs^{-1/2}(\Kmat^{\pseudinv})^{1/2} \fval \quad
  \Longleftrightarrow \quad \fval = \sqrt{\numobs} \Kmat^{1/2} \zval.
\end{align*}
as well as the new function $\EmpJloss(\zval) \defn
\EmpLoss(\sqrt{\numobs} \KmatSqrt \zval)$ and its population
equivalent $\Jloss(\zval) \defn \Exs \EmpJloss(\zval)$.  Ordinary
gradient descent on $\EmpJloss$ with stepsize $\stsize$ takes the form
\begin{align}
\label{EqnZupdate}
\zvalit{t+1} & = \zvalit{t} - \stsize \nabla
\EmpJloss(\zvalit{t}) \; = \; \zvalit{t} - \stsize \sqrt{\numobs}\KmatSqrt \nabla
\EmpLoss(\sqrt{\numobs}\KmatSqrt \zvalit{t}). 
\end{align}
If we transform this update on $\zval$ back to an equivalent one on
$\fval$ by multiplying both sides by $\sqrt{\numobs} \KmatSqrt$, we
see that ordinary gradient descent on $\EmpJloss$ is equivalent to
the kernel boosting update $\fvalit{t+1} = \fvalit{t} - \stsize
\numobs \Kmat \nabla \EmpLoss(\fvalit{t})$.

Our goal is to analyze the behavior of the update~\eqref{EqnZupdate}
in terms of the population cost $\Jloss(\zvalit{t})$.  Thus, our
problem is one of analyzing a noisy form of gradient descent on the
function $\Jloss$, where the noise is induced by the difference
between the empirical gradient operator $\nabla \EmpJloss$ and the
population gradient operator $\nabla \Jloss$.

Recall that the $\Loss$ is $\lipcon$-smooth by assumption. Since the
kernel matrix $\Kmat$ has been normalized to have largest eigenvalue
at most one, the function $\Jloss$ is also $\lipcon$-smooth, whence
\begin{align*}
\Jloss(\zvalit{t+1}) &\leq \Jloss(\zvalit{t}) + \inprod{\nabla
  \Jloss(\zvalit{t})}{\Sdiff{t}} + \frac{\lipcon}{2}
\|\Sdiff{t}\|_2^2, \\
&\mbox{where} \quad \Sdiff{t} \defn
\zvalit{t+1} - \zvalit{t} = - \stsize \nabla \EmpJloss(\zvalit{t}).
\end{align*}
Morever, since the function $\Jloss$ is convex, we have
$\Jloss(\zstar) \geq \Jloss(\zvalit{t}) + \inprod{\nabla
  \Jloss(\zvalit{t})}{\zstar - \zvalit{t}}$, whence
\begin{align}
  \Jloss(\zvalit{t+1}) - \Jloss(\zstar) & \leq \inprod{\nabla
    \Jloss(\zvalit{t})}{\Sdiff{t} + \zvalit{t} - \zstar} +
  \frac{\lipcon}{2} \|\Sdiff{t}\|_2^2 \nonumber \\
  \label{EqnAvocado}
& = \inprod{\nabla \Jloss(\zvalit{t})}{\zvalit{t+1} - \zstar} +
  \frac{\lipcon}{2} \|\Sdiff{t}\|_2^2.
\end{align}
Now define the difference of the squared errors $\Tele{t} \defn
\frac{1}{2} \Big \{ \|\zvalit{t} - \zstar\|_2^2 - \|\zvalit{t+1} -
\zstar\|_2^2 \Big \}$.  By some simple algebra, we have
\begin{align*}
\Tele{t} &\; = \; \frac{1}{2} \Big \{ \|\zvalit{t} - \zstar\|_2^2 -
\|\Sdiff{t} + \zvalit{t} - \zstar\|_2^2 \Big \} \\
 = &
-\inprod{\Sdiff{t}}{\zvalit{t} - \zstar} - \frac{1}{2}
\|\Sdiff{t}\|_2^2 \\
 = & - \inprod{\Sdiff{t}}{-\Sdiff{t} + \zvalit{t+1} - \zstar} -
  \frac{1}{2} \|\Sdiff{t}\|_2^2 \\
= & - \inprod{\Sdiff{t}}{\zvalit{t+1} - \zstar} + \frac{1}{2} \|\Sdiff{t}\|_2^2.
\end{align*}
Substituting back into equation~\eqref{EqnAvocado} yields
\begin{align*}
  \Jloss(\zvalit{t+1}) - \Jloss(\zstar) & \leq
  \frac{1}{\ \stsize}\Tele{t} + \inprod{\nabla \Jloss(\zvalit{t}) +
    \frac{\Sdiff{t}}{\stsize}}{\zvalit{t+1} - \zstar} \\
    &\; = \; \frac{1}{ \stsize} \Tele{t}
  + \inprod{\nabla \Jloss(\zvalit{t}) - \nabla
    \EmpJloss(\zvalit{t})}{\zvalit{t+1} - \zstar},
\end{align*}
where we have used the fact that $\frac{1}{ \stsize} \geq\lipcon$ by
our choice of stepsize $\stsize$.

Finally, we transform back to the original variables \mbox{$\fval
  =\sqrt{\numobs} \KmatSqrt \zval$,} using the relation \mbox{$\nabla
  \Jloss(\zval) = \sqrt{\numobs} \KmatSqrt \nabla \Loss(\fval)$, } so
as to obtain the bound
\begin{align*}
\Loss(\fvalit{t+1}) - \Loss(\fvalstar) \leq \frac{1}{2 \stsize} & \Big
\{ \HilNorm{\DelIt{t}}^2 - \HilNorm{\DelIt{t+1}}^2 \Big \} \\
& + \inprod{\nabla \Loss(\fvalit{t}) - \nabla
  \EmpLoss(\fvalit{t})}{\fvalit{t+1} - \fvalstar}.
\end{align*}
Note that the optimality of $\fvalstar$ implies that $\nabla
\Loss(\fvalstar) = 0$.  Combined with $\strongcon$-strong convexity,
we are guaranteed that $\frac{\strongcon}{2} \EmpNorm{\DelIt{t+1}}^2
\leq \Loss(\fvalit{t+1}) - \Loss(\fvalstar)$, and hence
\begin{align*}
\frac{\strongcon}{2} \enorm{\DelIt{t+1}}^2 \leq \frac{1}{2 \stsize} & \Big \{
\HilNorm{\DelIt{t}}^2 - \HilNorm{\DelIt{t+1}}^2 \Big \} \\
&+
\inprod{\nabla \Loss(\fvalstar + \DelIt{t}) - \nabla
  \EmpLoss(\fvalstar + \DelIt{t})}{\DelIt{t+1}},
\end{align*}
as claimed.  



\subsection{Proof of Lemma~\ref{LemPeeling}}
\label{SecProofPeeling}

We split our proof into two cases, depending on whether we are dealing
with the least-squares loss \mbox{$\phi(y, \theta) = \frac{1}{2} (y-
  \theta)^2$,} or a classification loss with uniformly bounded
gradient ($\|\phi'\|_\infty \leq 1$).


\subsubsection{Least-squares case}

The least-squares loss is $\strongcon$-strongly convex with
$\strongcon = \lipcon = 1$.  Moreover, the difference between the
population and empirical gradients can be written as $\nabla
\Loss(\fvalstar + \DelTil) - \nabla \EmpLoss(\fvalstar + \DelTil) =
\frac{\sigma}{\numobs}(w_1,\ldots,w_{\numobs})$, where the random
variables $\{w_i\}_{i=1}^n$ are i.i.d. and sub-Gaussian with parameter
$1$. Consequently, we have
\begin{align*}
|\inprod{\nabla \Loss(\fvalstar + \DelTil) - \nabla \EmpLoss(\fvalstar
  + \DelTil)}{\DelIt{}}| = \Biggr| \frac{\sigma}{\numobs} \sum_{i=1}^n
w_i \Delta(x_i) \Biggr|.
\end{align*}
Under these conditions, one can show (see \cite{Wai17} for reference) that
\begin{align}
\left| \frac{\sigma}{\numobs} \sum_{i=1}^n w_i \Delta(x_i) \right | &
\leq 2\delta_{\numobs} \enorm{\Delta} + 2\delta_{\numobs}^2
\HilNorm{\Delta} + \frac{1}{16} \enorm{\Delta}^2,
\end{align}
which implies that Lemma~\ref{LemPeeling} holds with $c_3 = 16$.


\subsubsection{Gradient-bounded $\phi$-functions}

We now turn to the proof of Lemma~\ref{LemPeeling} for gradient
bounded $\phi$-functions.  First, we claim that it suffices to prove
the bound~\eqref{EqnPeeling} for functions $g \in \Hilstar$ and
$\HilNorm{g} = 1$ where $\Hilstar \defn \{ f- g \mid f,g\in \Hil\}$.
Indeed, suppose that it holds for all such functions, and that we are
given a function $\Delta$ with $\HilNorm{\Delta} > 1$ . By assumption,
we can apply the inequality~\eqref{EqnPeeling} to the new function $g
\defn \Delta/\HilNorm{\Delta}$, which belongs to $\Hilstar$ by nature
of the subspace $\Hil =\widebar{\text{span}}\{\Kerfunc(\cdot,
x_i)\}_{i=1}^\numobs$.

Applying the bound~\eqref{EqnPeeling} to $g$ and then multiplying both
sides by $\HilNorm{\Delta}$, we obtain
\begin{align*}
  \inprod{\nabla \Loss(\fvalstar + \DelTil) &- \nabla
    \EmpLoss(\fvalstar + \DelTil)}{\Delta} \\
  \leq& 2\critquant
  \enorm{\Delta} + 2\critquant^2 \HilNorm{\Delta} + \frac{\strongcon}{ c_3}
  \frac{\enorm{\Delta}^2}{\HilNorm{\Delta}}\\
  \leq& 2\critquant \enorm{\Delta} + 2\critquant^2 \HilNorm{\Delta}+ \frac{\strongcon}{c_3}
  \enorm{\Delta}^2,
\end{align*}
where the second inequality uses the fact that $\HilNorm{\Delta} > 1$
by assumption.

In order to establish the bound~\eqref{EqnPeeling} for functions with
$\HilNorm{g} = 1$, we first prove it uniformly over the set $\{g \mid
\HilNorm{g} = 1, \quad \enorm{g} \leq t\}$, where $t > 1$ is a fixed
radius (of course, we restrict our attention to those radii $t$ for
which this set is non-empty.)  We then extend the argument to one that
is also uniform over the choice of $t$ by a ``peeling'' argument.


Define the random variable
\begin{align}
\label{EqnZn}
  \EGW(t) \defn \sup_{\Delta, \DelTil \in \Eset{t}{1}}\inprod{\nabla
    \Loss(\fvalstar + \DelTil) - \nabla \EmpLoss(\fvalstar +
    \DelTil)}{\Delta}.
\end{align}

The following two lemmas, respectively, bound the mean of
this random variable, and its deviations above the mean:

\begin{lems}
    \label{LemGW2RC}
  For any $t > 0$, the mean is upper bounded as
\begin{align}
  \label{EqnGW2RC}
    \Exs \EGW(t) \leq \level \GW(\Eset{t}{1}),
  \end{align}
where $\level \defn 2 \lipcon + 4 \Radius$.
\end{lems}

\begin{lems}
  \label{LemCorrectConc}
  There are universal constants $(c_1, c_2)$ such that
  \begin{align}
    \label{EqnCorrectConc}
  \Prob \Big[ \EGW(t) \geq \Exs \EGW(t) + \alpha \Big] & \leq c_1 \exp
  \Big( - \frac{c_2 \numobs \alpha^2}{ t^2} \Big).
  \end{align}
\end{lems}
\noindent See Appendices~\ref{AppLemGW2RC} and~\ref{AppLemCorrectConc}
for the proofs of these two claims.

Equipped with Lemmas~\ref{LemGW2RC} and~\ref{LemCorrectConc}, we now
prove inequality~\eqref{EqnPeeling}.  We divide our argument into
two cases:
\paragraph{Case $t=\critquant$}
We first prove inequality~\eqref{EqnPeeling} for $t=\critquant$.  From
Lemma~\ref{LemGW2RC}, we have
\begin{align}
  \Exs \EGW(\critquant) \leq \level \GW(\Eset{\critquant}{1})
  \stackrel{(i)}{\leq} \critquant^2,
\end{align}
where inequality (i) follows from the definition of $\critquant$ in
inequality~\eqref{Eqn75perChocolate}.  Setting $\alpha = \critquant^2$
in expression~\eqref{EqnCorrectConc} yields
\begin{align}
\label{EqnPrelude}
  \Prob \Big[ \EGW(\critquant) \geq 2 \critquant^2 \Big] \leq c_1 \exp
  \left( - c_2 n \critquant^2 \right),
\end{align}
which establishes the claim for $t = \critquant$.


\paragraph{Case $t > \critquant$}

On the other hand, for any $t > \critquant$, we have
\begin{align*}
\Exs \EGW(t) \stackrel{(i)}{\leq} \level \GW(\Eset{t}{1})
\stackrel{(ii)}{\leq} t \level \frac{\GW(\Eset{t}{1})}{t} \leq
t\critquant,
\end{align*}
where step (i) follows from Lemma~\ref{LemGW2RC}, and step (ii)
follows because the function $u \mapsto \frac{\GW(\Eset{u}{1})}{u}$ is
non-increasing on the positive real line.  (This non-increasing
property is a direct consequence of the star-shaped nature of
$\Hilstar$.)  Finally, using this upper bound on expression $\Exs
\EGW(\critquant)$ and setting $\alpha = t^2\strongcon/(4  c_3)$ in the tail
bound~\eqref{EqnCorrectConc} yields
\begin{align}
\label{EqnRising}
  \Prob \Big[ \EGW(t) \geq t \critquant + \frac{t^2\strongcon}{4 c_3} \Big] &
  \leq c_1 \exp \left( - c_2 \numobs \strongcon^2 t^2\right).
\end{align}
Note that the precise values of the universal constants $c_2$
may change from line to line throughout this section.


\paragraph{Peeling argument}

Equipped with the tail bounds~\eqref{EqnPrelude}
and~\eqref{EqnRising}, we are now ready to complete the peeling
argument.  Let $\A$ denote the event that the bound~\eqref{EqnPeeling}
is violated for some function $g \in \Hilstar$ with $\HilNorm{g} = 1$.
For real numbers $0 \leq a < b$, let $\A(a,b)$ denote the event that
it is violated for some function such that $\enorm{g} \in [a,b]$, and
$\HilNorm{g} = 1$. For $k = 0,1,2, \ldots$, define $t_k = 2^k
\critquant$. We then have the decomposition $\Event = (0,t_0) \cup
(\bigcup_{k=0}^{\infty} \A(t_k,t_{k+1}))$ and hence by union bound,
\begin{align}
\label{EqnDecom}
  \Prob[\Event] \leq \Prob[\A(0, \critquant)] + \sum_{k=1}^{\infty}
  \Prob[\A(t_k, t_{k+1})].
\end{align}

From the bound~\eqref{EqnPrelude}, we have $\mprob[\A(0, \critquant)]
\leq c_1 \exp \left( - c_2 n \critquant^2 \right)$.  On the other
hand, suppose that $\A(t_k,t_{k+1})$ holds, meaning that there exists
some function $g$ with $\HilNorm{g} = 1$ and $\enorm{g} \in [t_k,
  t_{k+1}]$ such that
\begin{align*}
  \inprod{\nabla \Loss(\fvalstar + \DelTil) - \nabla \EmpLoss(\fvalstar +
  \DelTil)}{g}
  &> 2\critquant \enorm{g}+ 2\critquant^2 + \frac{\strongcon}{ c_3} \enorm{g}^2\\
  &\stackrel{(i)}{\geq} 2\critquant t_k + 2\critquant^2 + \frac{\strongcon}{c_3} t_k^2\\
  &\stackrel{(ii)}{\geq} \critquant t_{k+1} + 2\critquant^2 + \frac{\strongcon}{4  c_3} t_{k+1}^2,
\end{align*}
where step (i) uses the $\enorm{g} \geq t_k$ and step (ii) uses the
fact that \mbox{$t_{k+1} = 2t_{k}.$} This lower bound implies that
$\EGW(t_{k+1}) > t_{k+1} \critquant + \frac{t_{k+1}^2 \strongcon}{4  c_3}$ and applying
the tail bound~\eqref{EqnRising} yields
\begin{align*}
\Prob(\A(t_k,t_{k+1})) &\leq \Prob(\EGW(t_{k+1}) > t_{k+1} \critquant +
  \frac{t_{k+1}^2\strongcon}{4 c_3}) \\
  &\leq \exp \left( - c_2 \numobs \strongcon^2
2^{2k+2} \critquant^2\right).
\end{align*}
Substituting this inequality and our earlier bound~\eqref{EqnPrelude}
into equation~\eqref{EqnDecom} yields
\begin{align*}
  \Prob(\Event) \leq c_1 \exp(-c_2 \numobs \strongcon^2 \critquant^2),
\end{align*}
where the reader should recall that the precise values of universal
constants may change from line-to-line. This concludes the proof
of Lemma~\ref{LemPeeling}.


\subsubsection{Proof of Lemma~\ref{LemGW2RC}}
\label{AppLemGW2RC}

Recalling the definitions~\eqref{EqnPopLoss} and~\eqref{EqnEmpLoss} of
$\Loss$ and $\EmpLoss$, we can write
\begin{align*}
\EGW(t) = \sup_{\Delta, \DelTil \in \Eset{t}{1}} \frac{1}{\numobs}
   \sum_{i=1}^{\numobs} (\phi'(y_i, \theta^*_i + \DelTil_i) - \Exs
   \phi'(y_i, \theta^*_i + \DelTil_i)) \Delta_i
\end{align*}
Note that the vectors $\Delta$ and $\DelTil$ contain function values
of the form $f(x_i) - \fstar(x_i)$ for functions $f \in
\HilNormBall(\fstar, 2 \Radius)$.  Recall that the kernel function is
bounded uniformly by one.  Consequently, for any function $f \in
\HilNormBall(\fstar, 2 \Radius)$, we have
\begin{align*}
|f(x) - f^*(x)| &= |\inprod{f - \fstar}{\Kerfunc(\cdot,x)}_\Hil| \\
&\leq \HilNorm{f - \fstar}\HilNorm{\Kerfunc(\cdot,x)} \leq 2\Radius.
\end{align*}
Thus, we can restrict our attention to vectors $\Delta, \DelTil$ with
$\|\Delta\|_\infty, \|\DelTil\|_\infty \leq 2 \Radius$ from
hereonwards.

Letting $\{\rade{i}\}_{i=1}^\numobs$ denote an i.i.d. sequence of
Rademacher variables, define the symmetrized variable
\begin{align}
  \label{EqnSym}
  \EGWS(t) & \defn \sup_{\Delta, \DelTil \in \Eset{t}{1}}
  \frac{1}{\numobs} \sum_{i=1}^{\numobs} \rade{i} \phi'(y_i,
  \theta^*_i + \DelTil_i) \; \Delta_i.
\end{align}
By a standard symmetrization argument~\cite{vanderVaart96}, we have
$\Exs_y[ \EGW(t)] \leq 2 \Exs_{y, \epsilon}[\EGWS(t)]$.  Moreover,
since
\begin{align*}
\phi'(y_i, \theta^*_i + \DelTil_i) \; \Delta_i \leq
\frac{1}{2}\Big(\phi'(y_i, \theta^*_i + \DelTil_i) \Big)^2 +
\frac{1}{2} \Delta^2_i
\end{align*}
we have
\begin{align*}
\Exs \EGW(t) \leq& \Exs \sup_{\DelTil \in \Eset{t}{1}}
  \frac{1}{\numobs} \sum_{i=1}^{\numobs} \rade{i} \big(\phi'(y_i,
  \theta^*_i + \DelTil_i) \big)^2 + \Exs \sup_{\Delta \in
    \Eset{t}{1}} \frac{1}{\numobs} \sum_{i=1}^{\numobs} \rade{i}
  \Delta^2_i \\
\leq& 2 \underbrace{\Exs \sup_{\DelTil \in \Eset{t}{1}}
    \frac{1}{\numobs} \sum_{i=1}^{\numobs} \rade{i} \phi'(y_i, \theta^*_i + \DelTil_i)}_{\TermOne} 
    + \, 4 \Radius
  \underbrace{\Exs \sup_{\Delta \in \Eset{t}{1}} \frac{1}{\numobs}
    \sum_{i=1}^{\numobs} \rade{i} \Delta_i}_{\TermTwo},
\end{align*}
where the second inequality follows by
applying the Rademacher contraction inequality~\cite{LedTal91}, using
the fact that $\|\phi'\|_\infty \leq 1$ for the first term, and
$\|\Delta\|_\infty \leq 2 \Radius$ for the second term.

Focusing first on the term $\TermOne$, since $\Exs[\rade{i} \phi'(y_i,
  \theta^*_i)] = 0$, we have
\begin{align*}
\TermOne & = \Exs \sup_{\DelTil \in \Eset{t}{1}}
\frac{1}{\numobs} \sum_{i=1}^{\numobs} \rade{i}
\underbrace{\Big(\phi'(y_i, \theta^*_i + \DelTil_i) - \phi'(y_i;
  \theta^*_i) \Big)}_{\varphi_i(\DelTil_i)} \\
& \stackrel{(i)}{\leq}  \lipcon \Exs \sup_{\DelTil \in \Eset{t}{1}}
\frac{1}{\numobs} \sum_{i=1}^{\numobs} \rade{i} \DelTil_i \\
& \stackrel{(ii)}{\leq} \sqrt{\frac{\pi}{2}} \lipcon \GW(\Eset{t}{1}),
\end{align*}
where step (i) follows since each function $\varphi_i$ is
$\lipcon$-Lipschitz by assumption; and step (ii) follows since the
Gaussian complexity upper bounds the Rademacher complexity up to a
factor of $\sqrt{\frac{\pi}{2}}$.  Similarly, we have
\begin{align*}
  \TermTwo & \leq \sqrt{\frac{\pi}{2}} \; \GW(\Eset{t}{1}),
\end{align*}
and putting together the pieces yields the claim.


\subsubsection{Proof of Lemma~\ref{LemCorrectConc}}
\label{AppLemCorrectConc}

Recall the definition~\eqref{EqnSym} of the symmetrized variable
$\EGWS$.  By a standard symmetrization argument~\cite{vanderVaart96},
there are universal constants $c_1, c_2$ such that
\begin{align*}
  \mprob \Big[ \EGW(t) \geq \Exs \EGW[t] + c_1 \alpha \Big] & \leq c_2
  \mprob \Big[ \EGWS(t) \geq \Exs \EGWS[t] + \alpha \Big].
\end{align*}
Since $\{\rade{i}\}_{i=1}^\numobs$ are $\{y_i\}_{i=1}^\numobs$ are
independent, we can study $\EGWS(t)$ conditionally on
$\{y_i\}_{i=1}^\numobs$.  Viewed as a function of
$\{\rade{i}\}_{i=1}^\numobs$, the function $\EGWS(t)$ is convex and
Lipschitz with respect to the Euclidean norm with parameter
\begin{align*}
L^2 & \defn \sup_{\Delta, \DelTil \in \Eset{t}{1}} \frac{1}{\numobs^2}
\sum_{i=1}^\numobs \Big( \phi'(y_i, \theta^*_i + \DelTil_i) \;
\Delta_i \Big)^2 \; \leq \; \frac{t^2}{\numobs},
\end{align*}
where we have used the facts that $\|\phi'\|_\infty \leq 1$ and
$\|\Delta\|_\numobs \leq t$.  By Ledoux's concentration for convex
and Lipschitz functions~\cite{Ledoux01}, we have
\begin{align*}
  \mprob \Big[ \EGWS(t) \geq \Exs \EGWS[t] + \alpha \; \mid \;
    \{y_i\}_{i=1}^\numobs \Big] & \leq c_3 \exp \Big(- c_4
  \frac{\numobs \alpha^2}{t^2} \Big).
\end{align*}
Since the right-hand side does not involve $\{y_i\}_{i=1}^\numobs$,
the same bound holds unconditionally over the randomness in both the
Rademacher variables and the sequence $\{y_i\}_{i=1}^\numobs$.
Consequently, the claimed bound~\eqref{EqnCorrectConc} follows, with
suitable redefinitions of the universal constants.


\subsection{Proof of Lemma~\ref{LemHilBound}}
\label{SecLemHilProof}

We first require an auxiliary lemma, which we state and prove in the
following section.  We then prove Lemma~\ref{LemHilBound} in
Section~\ref{SecLemHilProofSub}.


\subsubsection{An auxiliary lemma}

The following result relates the Hilbert norm of the error to the
difference between the empirical and population gradients:
\begin{lems}
\label{LemZurich}
For any convex and differentiable loss function $\Loss$, the kernel
boosting error \mbox{$\DelIt{t+1} \defn \fval^{t+1} - \fvalstar$}
satisfies the bound
\begin{align}
  \label{EqnZurich}
  \notag \HilNorm{\DelIt{t+1}}^2 \leq \HilNorm{\DelIt{t}} &\HilNorm{\DelIt{t+1}} \\
  & + \stsize \inprod{\nabla \Loss(\fvalstar + \DelIt{t}) - \nabla
    \EmpLoss(\fvalstar + \DelIt{t})}{\DelIt{t+1}}.
\end{align}
\end{lems}
\begin{proof}

Recall that $\HilNorm{\DelIt{t}}^2 = \HilNorm{\fval^t - \fvalstar}^2 =
\ltwo{\zval^t - \zvalstar}^2$ by definition of the Hilbert norm.  Let
us define the population update operator $\PopOp$ on the population
function $\Jloss$ and the empirical update operator $\EmpOp$ on
$\EmpJloss$ as
\begin{align}
  \label{EqnPopOp}
\notag \PopOp(\zval^t) \defn \zval^t - \stsize \nabla \Jloss (\sqrt{\numobs}
\KmatSqrt \zval^t), \\
\mbox{and} \quad \zval^{t+1} \defn
\EmpOp(\zval^t) = \zval^t - \stsize \nabla \EmpJloss (\sqrt{\numobs}
\KmatSqrt \zval^t).
\end{align}
Since $\Jloss$ is convex and smooth, it follows from standard
arguments in convex optimization that $\PopOp$ is a non-expansive
operator---viz.
\begin{align}
  \label{EqnPopOpNonExpansive}
\ltwo{\PopOp(x) - \PopOp(y)} \leq \ltwo{x- y} \qquad \mbox{for all
  $x,y \in \ConstSet$.}
\end{align}
In addition, we note that the vector $\zvalstar$ is a fixed point of
$\PopOp$---that is, $\PopOp(\zvalstar) = \zvalstar$.  From these
ingredients, we have
\begin{align*}
& \HilNorm{\DelIt{t+1}}^2 \\
= \,& \inprod{ \zval^{t+1} - \zvalstar}{
  \EmpOp(\zval^t) - \PopOp(\zval^t) + \PopOp(\zval^t) - \zvalstar}
\\
 \stackrel{(i)}{\leq} \,& \ltwo{\zval^{t+1} - \zvalstar}
\ltwo{\PopOp(\zval^t) - \PopOp(\zvalstar)} \\
& \,+ \stsize \inprod{
  \sqrt{\numobs} \KmatSqrt [\nabla\Loss(\fvalstar + \DelIt{t}) -
    \nabla\EmpLoss (\fvalstar + \DelIt{t})]}{\zval^{t+1} - \zvalstar}
\\
 \stackrel{(ii)}{\leq}\,& \HilNorm{\DelIt{t+1}} \HilNorm{\DelIt{t}}  \\
& + \stsize \inprod{\nabla\Loss(\fvalstar + \DelIt{t}) - \nabla\EmpLoss
  (\fvalstar + \DelIt{t})}{\DelIt{t+1}}
\end{align*}
where step (i) follows by applying the Cauchy-Schwarz to control the
inner product, and step (ii) follows since $\DelIt{t+1} =
\sqrt{\numobs} \KmatSqrt (\zval^{t+1} - \zvalstar)$, and the square
root kernel matrix $\KmatSqrt$ is symmetric.
\end{proof}


\subsubsection{Proof of Lemma~\ref{LemHilBound}}
\label{SecLemHilProofSub}

We now prove Lemma~\ref{LemHilBound}.  The argument makes use of
Lemmas~\ref{LemMasterBound} and~\ref{LemPeeling} combined with
Lemma~\ref{LemZurich}.

In order to prove inequality~\eqref{EqnDeltaStayBounded}, we follow an
inductive argument. Instead of proving~\eqref{EqnDeltaStayBounded}
directly, we prove a slightly stronger relation which implies it,
namely
\begin{align}
\label{EqnIndStep}
\max\{1, \HilNorm{\DelIt{t}}^2\} & \leq \max\{1,
\HilNorm{\DelIt{0}}^2\} + t \critquant^2 \frac{4 \lipcon}{\scefftil
  \strongcon}.
\end{align}
Here $\scefftil$ and $\AnnCon$ are constants linked by the relation
\begin{align}
  \label{EqnGamma}
  \scefftil \defn
  \frac{1}{32} - \frac{1}{4 \AnnCon} = 1/\Radius^2.
\end{align}
We claim that it suffices to prove that the error iterates
$\DelIt{t+1}$ satisfy the inequality~\eqref{EqnIndStep}.  Indeed, if
we take inequality~\eqref{EqnIndStep} as given, then we have
\begin{align*}
  \HilNorm{\DelIt{t}}^2 \leq \max\{1, \HilNorm{\DelIt{0}}^2 \} +
  \frac{1}{2\scefftil } \leq \Radius^2,
\end{align*}
where we used the definition $\Radius^2 =
2\max\{\HilNorm{\fvalstar}^2,~ 32\}$.  Thus, it suffices to focus our
attention on proving inequality~\eqref{EqnIndStep}.

For $t = 0$, it is trivially true.  Now let us assume
inequality~\eqref{EqnIndStep} holds for some $t \leq
\frac{\strongcon}{8\lipcon \critquant^2}$, and then prove that it also
holds for step $t+1$.


If $\HilNorm{\DelIt{t+1}} < 1$, then inequality~\eqref{EqnIndStep}
follows directly.  Therefore, we can assume without loss of generality
that $\HilNorm{\DelIt{t+1}} \geq 1$.

We break down the proof of this induction into two steps:
\bcar
\item First, we show that $\HilNorm{\DelIt{t+1}} \leq 2 \Radius$ so
  that Lemma~\ref{LemPeeling} is applicable.
\item Second, we show that the bound~\eqref{EqnIndStep} holds and thus
  in fact $\HilNorm{\DelIt{t+1}} \leq \Radius$.
\ecar
  
Throughout the proof, we condition on the event $\mathcal{E}$ and
$\mathcal{E}_0 := \{ \frac{1}{\sqrt{\numobs}}\ltwo{y-\Exs [y \mid x]}
\leq \sqrt{2}\sigma\}$. Lemma~\ref{LemPeeling} guarantees that
$\Prob(\Event^c) \leq c_1 \exp(-c_2 \frac{\strongcon^2 \numobs
  \delta_\numobs^2}{\level^2})$ whereas $\Prob (\mathcal{E}_0) \geq 1-
\E^{-n}$ follows from the fact that $Y^2$ is sub-exponential with
parameter $\sigma^2 \numobs$ and applying Hoeffding's inequality.
Putting things together yields an upper bound on the probability of
the complementary event, namely
  \begin{equation*}
    \Prob(\Event^c \cup \Event_0^c) \leq 2c_1 \exp(-C_2 \numobs
    \delta_\numobs^2)
  \end{equation*}
with $C_2 = \max\{ \frac{\strongcon^2}{\level^2} , 1\}$.

  
\paragraph{Showing that $\HilNorm{\DelIt{t+1}} \leq 2 \Radius$}

In this step, we assume that inequality~\eqref{EqnIndStep} holds at
step $t$, and show that $\HilNorm{\DelIt{t+1}} \leq 2 \Radius$.
Recalling that $\zval \defn
\frac{(\Kmat^{\pseudinv})^{1/2}}{\sqrt{\numobs}} \fval$, our update
can be written as
\begin{align*}
\zvalit{t+1} - \zvalstar &= \zvalit{t} - \stsize
\sqrt{\numobs}\KmatSqrt \nabla \Loss(\fvalit{t}) - \zvalstar \\
& \hspace*{1cm} + \stsize \sqrt{\numobs} \KmatSqrt (\nabla \EmpLoss(\fvalit{t}) - \nabla
\Loss(\fvalit{t})).
\end{align*}
Applying the triangle inequality yields the bound
\begin{align*}
\ltwo{\zvalit{t+1} - \zvalstar} &\leq \ltwo{\underbrace{\zvalit{t} -
    \stsize \sqrt{\numobs}\KmatSqrt \nabla
    \Loss(\fvalit{t})}_{\PopOp(\zvalit{t})} - \zvalstar} \\
    & \hspace*{1cm }+ \ltwo{\stsize \sqrt{\numobs} \KmatSqrt (\nabla \EmpLoss(\fvalit{t}) -
  \nabla \Loss(\fvalit{t}))}
\end{align*}
where the population update operator $\PopOp$ was previously
defined~\eqref{EqnPopOp}, and observed to be
non-expansive~\eqref{EqnPopOpNonExpansive}.  From this non-expansiveness,
we find that
\begin{align*}
\ltwo{\zvalit{t+1} - \zvalstar} & \leq \ltwo{\zvalit{t} - \zvalstar} +
\ltwo{\stsize \sqrt{\numobs} \KmatSqrt (\nabla \EmpLoss(\fvalit{t}) -
  \nabla \Loss(\fvalit{t}))},
\end{align*}

Note
that the $\ell_2$ norm of $\zvalit{}$ corresponds to the Hilbert norm
of $\fvalit{}$.  This implies
\begin{align*}
  \HilNorm{\DelIt{t+1}} &\leq \HilNorm{\DelIt{t}} +
  \underbrace{\ltwo{\stsize \sqrt{\numobs} \KmatSqrt (\nabla
      \EmpLoss(\fvalit{t}) - \nabla \Loss(\fvalit{t}))}}_{\defn
    \CrudeTerm}
\end{align*}
Observe that because of uniform boundedness of the kernel by one, the
quantity $\CrudeTerm$ can be bounded as
\begin{align*}
  \CrudeTerm \leq \stsize \sqrt{\numobs} \ltwo{\nabla
    \EmpLoss(\fvalit{t}) - \nabla \Loss(\fvalit{t}))} = \stsize
  \sqrt{\numobs} \frac{1}{\numobs} \ltwo{v - \Exs v},
\end{align*}
where we have define the vector $v \in \real^\numobs$ with coordinates
$v_i \defn \phi'(y_i, \fvalit{t}_i)$. For
functions $\phi$ satisfying the gradient boundedness and $m-M$
condition, since
$\fval^t \in \HilNormball(\fvalstar, \Radius)$, each coordinate of the
vectors $v$ and $\Exs v$ is bounded by $1$ in absolute value.
We consequently have
\begin{align*}
\CrudeTerm \leq \stsize \leq \Radius,
\end{align*}
where we have used the fact that $\stsize \leq \strongcon/\lipcon < 1 \leq \frac{\Radius}{2}.$
For least-squares $\phi$ we instead have
\begin{align*}
  T \leq \alpha \frac{\sqrt{\numobs} }{\numobs} \ltwo{y-\Exs [y \mid x]} =: \frac{\alpha}{\sqrt{\numobs}} Y \leq \sqrt{2}\sigma \leq  \Radius
\end{align*}
conditioned on the event $\mathcal{E}_0 := \{ \frac{1}{\sqrt{\numobs}}\ltwo{y-\Exs [y \mid x]} \leq  \sqrt{2}\sigma\}$. Since $Y^2$ is sub-exponential with parameter $\sigma^2 \numobs$ it follows by Hoeffding's inequality that  $\Prob (\mathcal{E}_0) \geq 1- \E^{-n}$.


Putting together the pieces yields that $\HilNorm{\DelIt{t+1}} \leq 2
\Radius$, as claimed.


\paragraph{Completing the induction step}

We are now ready to complete the induction step for proving
inequality~\eqref{EqnIndStep} using Lemma~\ref{LemMasterBound} and
Lemma~\ref{LemPeeling} since $\HilNorm{\DelIt{t+1}}\geq 1$.
We split the argument into two
cases separately depending on whether or not
$\HilNorm{\DelIt{t+1}} \critquant \geq \enorm{\DelIt{t+1}}$. In
general we can assume that
$\HilNorm{\DelIt{t+1}} > \HilNorm{\DelIt{t}}$, otherwise the induction
inequality~\eqref{EqnIndStep} satisfies trivially.

\paragraph{Case 1}
When $\HilNorm{\DelIt{t+1}} \critquant \geq \enorm{\DelIt{t+1}}$,
inequality~\eqref{EqnPeeling} implies that
\begin{align}
\label{EqnSurface}
\notag \inprod{\nabla \Loss(\fvalstar + \widetilde{\Delta}) - \nabla
   \EmpLoss(\fvalstar + & \widetilde{\Delta})}{\DelIt{t+1}}
\\
     &\leq 4\delta_{\numobs}^2 \HilNorm{\DelIt{t+1}} + \frac{\strongcon}{ \AnnCon}
 \enorm{\DelIt{t+1}}^2,
\end{align}
Combining Lemma~\ref{LemZurich} and inequality
\eqref{EqnSurface},
we obtain
\begin{align}
\label{EqnKorean}
\notag \HilNorm{\DelIt{t+1}}^2 
\leq &\HilNorm{\DelIt{t}}
\HilNorm{\DelIt{t+1}} + 4 \stsize\critquant^2 \HilNorm{\DelIt{t+1}} +
\stsize \frac{\strongcon}{\AnnCon} \enorm{\DelIt{t+1}}^2\\
\implies & \HilNorm{\DelIt{t+1}} \leq \frac{1}{1 -\stsize
  \critquant^2\frac{\strongcon}{\AnnCon} } \big[ \HilNorm{\Delta^t} +
  4 \stsize \critquant^2 \big],
\end{align}
where the last inequality uses the fact that $\enorm{\DelIt{t+1}} \leq \critquant
\HilNorm{\DelIt{t+1}}.$

\vspace*{0.2cm}

\paragraph{Case 2}
When $\HilNorm{\DelIt{t+1}} \critquant < \enorm{\DelIt{t+1}}$, we use
our assumption $\HilNorm{\Delta^{t+1}} \geq \HilNorm{\Delta^t}$
together with Lemma~\ref{LemZurich} and inequality \eqref{EqnPeeling}
which guarantee that
\begin{align*}
\HilNorm{\DelIt{t+1}}^2
\leq &\HilNorm{\DelIt{t}}^2 + 2 \stsize
\inprod{\nabla \Loss(\fvalstar + \DelIt{t}) - \nabla
  \EmpLoss(\fvalstar + \DelIt{t})}{\DelIt{t+1}} \nonumber\\
 \leq &\HilNorm{\DelIt{t}}^2 + 8 \stsize \critquant
\enorm{\DelIt{t+1}} + 2 \stsize \frac{\strongcon}{\AnnCon}
\enorm{\DelIt{t+1}}^2 \nonumber.
\end{align*}
Using the elementary inequality $2ab \leq a^2+b^2$, we find that
\begin{align}
\label{EqnBiking}
\notag \HilNorm{\DelIt{t+1}}^2 \leq  &\HilNorm{\DelIt{t}}^2 + 8
\stsize \left[ \strongcon \scefftil \enorm{\DelIt{t+1}}^2 + \frac{1}{4
    \scefftil \strongcon} \critquant^2\right] + 2 \stsize
\frac{\strongcon}{ \AnnCon} \enorm{\DelIt{t+1}}^2 \\
\leq &\HilNorm{\DelIt{t}}^2 + \stsize \frac{\strongcon}{4 }
\enorm{\DelIt{t+1}}^2 + \frac{2\stsize \critquant^2 }{\scefftil
  \strongcon},
\end{align}
where in the final step, we plug in the constants $\scefftil,\AnnCon$
which satisfy equation~\eqref{EqnGamma}.

Now Lemma~\ref{LemMasterBound} implies that
\begin{align*}
\frac{\strongcon}{2}\enorm{\DelIt{t+1}}^2 &\leq \Hdiff{t} + 4
\enorm{\DelIt{t+1}} \critquant + \frac{\strongcon}{ \AnnCon}
\enorm{\DelIt{t+1}}^2 \\ &\stackrel{(i)}{\leq} \Hdiff{t} + 4 \left[
  \scefftil \strongcon \enorm{\DelIt{t+1}}^2 + \frac{1}{4
    \scefftil\strongcon} \critquant^2 \right] +
\frac{\strongcon}{ \AnnCon} \enorm{\DelIt{t+1}}^2,
\end{align*}
where step (i) again uses $2ab \leq a^2+b^2$.  Thus, we have
$\frac{\strongcon}{4} \enorm{\DelIt{t+1}}^2 \leq \Hdiff{t} +
\frac{1}{\scefftil \strongcon} \critquant^2$.  Together with
expression~\eqref{EqnBiking}, we find that
\begin{align}
\label{EqnOpera}
\notag \HilNorm{\DelIt{t+1}}^2 & \leq \HilNorm{\DelIt{t}}^2 +
\frac{1}{2} (\HilNorm{\DelIt{t}}^2 - \HilNorm{\DelIt{t+1}}^2) +
\frac{4\stsize }{\scefftil \strongcon} \critquant^2\\
&\implies \HilNorm{\DelIt{t+1}}^2 \leq \HilNorm{\DelIt{t}}^2 +
\frac{4\stsize }{ \scefftil \strongcon} \critquant^2.
\end{align}


\vspace*{0.2cm}
\paragraph{Combining the pieces} By combining the two previous
cases, we arrive at the bound 
\begin{align}
  \label{EqnOryx}
\notag &\max \Big \{1, \HilNorm{\DelIt{t+1}}^2 \Big \} \\
 \leq &\max \Big\{ 1,
\kappa^2 ( \HilNorm{\DelIt{t}} + 4 \stsize \critquant^2)^2,
\HilNorm{\DelIt{t}}^2 + \frac{4 \lipcon }{\scefftil \strongcon}
\critquant^2 \Big\},
\end{align} 
where
$\kappa \defn \frac{1}{(1- \stsize \critquant^2
  \frac{\strongcon}{\AnnCon})}$ and we used that $\alpha \leq \min\{\frac{1}{\lipcon}, \lipcon\}$.

Now it is only left for us to show that with the constant $c_3$ chosen
such that $\scefftil = \frac{1}{32} - \frac{1}{4 \AnnCon} =
1/\Radius^2$, we have
\begin{align*}
\kappa^2 (\HilNorm{\DelIt{t}} + 4 \stsize \critquant^2)^2 &\leq
\HilNorm{\DelIt{t}}^2+\frac{4\lipcon}{\scefftil \strongcon} \critquant^2.
\end{align*}

Define the function $f: (0,\Radius] \rightarrow \real$ via
$f(\matchquant) \defn \kappa^2 (\matchquant + 4 \stsize
\critquant^2)^2 - \matchquant^2 - \frac{4 \lipcon}{\scefftil
  \strongcon} \critquant^2$.  Since $\kappa \geq 1$, in order to
conclude that $f(\xi) <0$ for all $\xi \in (0,\Radius]$, it suffices
to show that $\argmin_{x\in\real} f(x) < 0$ and $f(\Radius)<0$.
The former is obtained by basic algebra and follows directly from $\kappa \geq 1$. 
For the latter, since
$\scefftil = \frac{1}{32} - \frac{1}{4 \AnnCon} =
1/\Radius^2$, $\alpha < \frac{1}{\lipcon}$ and $\critquant^2 \leq \frac{\lipcon^2}{\strongcon^2}$ it thus suffices to show
\begin{equation*}
  \frac{1}{(1- \frac{\lipcon}{8\strongcon} )^2} \leq \frac{4\lipcon}{\strongcon} + 1
\end{equation*}
Since $  (4x + 1)(1 - \frac{x}{8})^2 \geq 1$ for all $x \leq 1$ and $\frac{\strongcon}{\lipcon}\leq 1$,
we conclude that $f(\Radius) < 0$.

Now that we have established
$\max\{1, \HilNorm{\DelIt{t+1}}^2\} \leq \max\{ 1,
\HilNorm{\DelIt{t}}^2 \}+ \frac{4 \lipcon }{\scefftil \strongcon}
\critquant^2$, the induction step~\eqref{EqnIndStep} follows. 
which completes the proof of Lemma~\ref{LemHilBound}.


\subsection{Proof of Lemma~\ref{LemLogAda}}
\label{SecProofLogAda}

Recall that the LogitBoost algorithm is based on logistic loss $\phi(y,
\secarg) = \ln(1+ e^{-y \secarg })$, whereas the AdaBoost algorithm is
based on the exponential loss $\phi(y, \secarg) = \exp( -y
\secarg)$. We now verify the \mMC~for these two losses with the
corresponding parameters specified in Lemma~\ref{LemLogAda}.


\subsubsection{\mMC~for logistic loss}

The first and second derivatives are given by
\begin{align*}
  \frac{\partial \phi(y, \fval)}{\partial \fval} = \frac{-y e^{- y
      \fval}}{1+ e^{- y \fval}}, \qquad \text{ and } \quad
  \frac{\partial^2 \phi(y, \fval)}{(\partial \fval)^2} =
  \frac{y^2}{(e^{-y\fval/2}+ e^{y\fval/2})^2}.
\end{align*}
It is easy to check that $|\frac{\partial \phi(y, \fval)}{\partial
  \fval}|$ is uniformly bounded by $\PhiB = 1$.

Turning to the second derivative, recalling that $y \in \{-1, +1 \}$,
it is straightforward to show that
\begin{align*}
\max_{y \in \{-1, +1 \} } \sup_{\fval} \frac{y^2}{(e^{-y\fval/2}+
  e^{y\fval/2})^2} \leq \frac{1}{4},
\end{align*}
which implies that $\frac{\partial \phi(y,\fval)}{\partial \fval}$ is
a $1/4$-Lipschitz function of $\fval$, i.e. with $\lipcon = 1/4$.

Our final step is to compute a value for $\strongcon$ by deriving a
uniform lower bound on the Hessian.  For this step, we need to exploit
the fact that $\fval = f(x)$ must arise from a function $f$ such that
\mbox{$\HilNorm{f} \leq \Diam \defn \Radius + \HilNorm{\fvalstar}$.}
Since $\sup_{x} \Kerfunc(x,x) \leq 1$ by assumption, the reproducing
relation for RKHS then implies that \mbox{$|f(x)| \leq \Diam$.}
Combining this inequality with the fact that $y \in \{-1, 1 \}$, it
suffices to lower the bound the quantity
\begin{align*}
  \min_{y \in \{-1, +1 \}} \min_{|\fval| \leq \Diam}
  \left|\frac{\partial^2 \phi(y, \fval)}{(\partial \fval)^2}\right| &=
  \min_{|y| \leq 1} \min_{|\fval| \leq \Diam}
  \frac{y^2}{(e^{-y\fval/2}+ e^{y\fval/2})^2} \\
  &\geq \underbrace{\frac{1}{e^{-D} + e^D + 2}}_{\strongcon},
\end{align*}
which completes the proof for the logistic loss.


\subsubsection{\mMC~for AdaBoost}
The AdaBoost algorithm is based on the cost function $\phi(y, \fval) =
e^{-y \fval}$, which has first and second derivatives (with respect to
its second argument) given by
\begin{align*}
  \frac{\partial\phi(y, \fval)}{\partial \fval} =  -y e^{-y \fval},
  \qquad
  \text{ and }~~\frac{\partial^2\phi(y, \fval)}{(\partial \fval)^2}  =  e^{-y \fval}.
\end{align*}
As in the preceding argument for logistic loss, we have the bound $|y|
\leq 1$ and $|\fval| \leq \Diam$.  By inspection, the absolute value
of the first derivative is uniformly bounded $\PhiB \defn e^{\Diam}$,
whereas the second derivative always lies in the interval
$[\strongcon, \lipcon]$ with $\lipcon \defn e^\Diam$ and $\strongcon
\defn e^{-\Diam}$, as claimed.

\end{document}

%% file: BookMacros
\newcommand{\bcar}{\begin{itemize}}
\newcommand{\ecar}{\end{itemize}}

\newcommand{\Prob}{\ensuremath{\mprob}}

\newcommand{\Yspace}{\ensuremath{\mathcal{Y}}}

\newcommand{\DelTil}{\ensuremath{\widetilde{\Delta}}}

\newcommand{\TermOne}{\ensuremath{\Term_1}}
\newcommand{\TermTwo}{\ensuremath{\Term_2}}

\newcommand{\strongcon}{\ensuremath{\gamma}}

\newcommand{\thetastar}{\ensuremath{{\theta^*}}}

\newcommand{\NORMAL}{\ensuremath{\mathcal{N}}}

\newcommand{\thetahat}{\ensuremath{\widehat{\theta}}}

\newcommand{\Ellipse}{\ensuremath{\mathcal{E}}}

\definecolor{MyGray}{rgb}{0.9,0.9,0.9}

\makeatletter\newenvironment{graybox}{ 
\begin{lrbox}{\@tempboxa}
\begin{minipage}{1\columnwidth}}{\end{minipage}
\end{lrbox}%
\colorbox{MyGray}{\usebox{\@tempboxa}} }
\makeatother



\newcommand{\Zback}[1]{\ensuremath{Z^{\backslash i}}}

\newcommand{\argmin}{\ensuremath{\operatorname{argmin}}}

\newcommand{\range}{\ensuremath{\operatorname{range}}}

\newcommand{\xsamstack}[1]{\ensuremath{x_1^\numobs}}
\newcommand{\Xsamstack}[1]{\ensuremath{X_1^\numobs}}

\newcommand{\widgraph}[2]{\includegraphics[keepaspectratio,width=#1]{#2}}

\newcommand{\muhat}{\ensuremath{\widehat{\mu}}}

\newcommand{\Term}{\ensuremath{T}}

\newcommand{\radepl}{\ensuremath{\varepsilon}}
\newcommand{\rade}[1]{\ensuremath{\radepl_{#1}}}

\newcommand{\Event}{\ensuremath{\mathcal{E}}}

\newcommand{\Fclass}{\ensuremath{\mathscr{F}}}

\newcommand{\fancysoln}[1]{
\ifthenelse{\equal{\doctype}{WITHSOLS}}
{
\begin{soln}
#1
\end{soln}
}
{
}
}

\newcommand{\fstar}{\ensuremath{f^*}}

\long\def\comment#1{}

\makeatletter
\def\@cite#1#2{[\if@tempswa #2 \fi #1]}
\makeatother

\newcommand{\Hil}{\ensuremath{\mathbb{H}}}

\newcommand{\fhat}{\ensuremath{\widehat{f}}}

\newcommand{\Xspace}{\ensuremath{\mathcal{X}}}

\newcommand{\defn}{\ensuremath{: \, = }}

\newcommand{\real}{\ensuremath{\mathbb{R}}}

\newcommand{\numobs}{\ensuremath{n}}

\newcommand{\mprob}{\ensuremath{\mathbb{P}}}

\newcommand{\Kmat}{\ensuremath{\mymathbf{K}}}

\newcommand{\Loss}{\ensuremath{\LossSing_\numobs}}

\newlength{\widebarargwidth}
\newlength{\widebarargheight}
\newlength{\widebarargdepth}
\DeclareRobustCommand{\widebar}[1]{%
  \settowidth{\widebarargwidth}{\ensuremath{#1}}%
  \settoheight{\widebarargheight}{\ensuremath{#1}}%
  \settodepth{\widebarargdepth}{\ensuremath{#1}}%
  \addtolength{\widebarargwidth}{-0.3\widebarargheight}%
  \addtolength{\widebarargwidth}{-0.3\widebarargdepth}%
  \makebox[0pt][l]{\hspace{0.3\widebarargheight}%
    \hspace{0.3\widebarargdepth}%
    \addtolength{\widebarargheight}{0.3ex}%
    \rule[\widebarargheight]{0.95\widebarargwidth}{0.1ex}}%
  {#1}}

\newcommand{\inprod}[2]{\ensuremath{\langle #1 , \, #2 \rangle}}

\newcommand{\Exs}{\ensuremath{\mathbb{E}}}

%% file: AMS_commands.tex
\theoremstyle{plain}

\newtheorem{theo}{Theorem}[section]

\newtheorem{lem}{Lemma}[section]
\newtheorem{prop}{Proposition}[section]
\newtheorem{cor}{Corollary}[section]

\theoremstyle{definition} 

\newtheorem{nota}{Notation}[section]
\newtheorem{de}{Definition}[section]
\newtheorem{exa}{Example}[section]
\newtheorem{as}{Assumption}[section]
\newtheorem{alg}{Algorithm}[section]

\newcommand{\btheo}{\begin{theo}}
\newcommand{\bde}{\begin{de}}
\newcommand{\ble}{\begin{lem}}
\newcommand{\bpr}{\begin{prop}}
\newcommand{\bno}{\begin{nota}}
\newcommand{\bex}{\begin{exa}}
\newcommand{\bcor}{\begin{cor}}
\newcommand{\spro}{\begin{proof}}
\newcommand{\bas}{\begin{as}}
\newcommand{\balg}{\begin{alg}}

\newcommand{\etheo}{\end{theo}}
\newcommand{\ede}{\end{de}}
\newcommand{\ele}{\end{lem}}
\newcommand{\epr}{\end{prop}}
\newcommand{\eno}{\end{nota}}
\newcommand{\eex}{\end{exa}}
\newcommand{\ecor}{\end{cor}}
\newcommand{\fpro}{\end{proof}}
\newcommand{\eas}{\end{as}}
\newcommand{\ealg}{\end{alg}}

\theoremstyle{plain}

\newtheorem{theos}{Theorem}
\newtheorem{props}{Proposition}
\newtheorem{lems}{Lemma}
\newtheorem{cors}{Corollary}

\theoremstyle{definition}
\newtheorem{exas}{Example}
\newtheorem{algs}{Algorithm}
\newtheorem{asss}{Assumption}
\newtheorem{defns}{Definition}

\newcommand{\btheos}{\begin{theos}}
\newcommand{\etheos}{\end{theos}}
\newcommand{\brems}{\begin{remark}}
\newcommand{\erems}{\end{remark}}
\newcommand{\bprops}{\begin{props}}
\newcommand{\eprops}{\end{props}}
\newcommand{\bdes}{\begin{defns}}
\newcommand{\edes}{\end{defns}}
\newcommand{\blems}{\begin{lems}}
\newcommand{\elems}{\end{lems}}
\newcommand{\bcors}{\begin{cors}}
\newcommand{\ecors}{\end{cors}}
\newcommand{\bexs}{\begin{exas}}
\newcommand{\eexs}{\end{exas}}
\newcommand{\balgs}{\begin{algs}}
\newcommand{\ealgs}{\end{algs}}
\newcommand{\bass}{\begin{asss}}
\newcommand{\eass}{\end{asss}}